\documentclass[10pt,twocolumn,letterpaper]{article}

\usepackage[pagenumbers]{cvpr} 

\newcommand{\multiviz}{\textsc{MultiViz}\xspace}
\newcommand{\multishap}{\textsc{MultiSHAP}\xspace}
\newcommand{\emap}{\textsc{Emap}\xspace}
\newcommand{\dime}{\textsc{Dime}\xspace}
\newcommand{\xhdr}[1]{\vspace{0em}\noindent{{\bf #1.}}}
\newcommand{\err}[1]{\ensuremath{\scriptscriptstyle \pm\,#1}}

\usepackage{microtype}

\renewcommand{\paragraph}[1]{\vspace{.5em}\noindent\textbf{#1.}}

\usepackage{xspace}
\newcommand{\method}{GridVQA-X\xspace}

\usepackage{booktabs}
\usepackage{pifont}
\usepackage{xcolor}
\usepackage{algorithm}
\usepackage{algpseudocode}

\usepackage{stfloats}
\usepackage{caption}
\usepackage{amsthm}
\usepackage{multirow}

\usepackage{soul}
\setuldepth{foobar}

\newtheorem{definition}{Definition}
\newtheorem{theorem}{Theorem}

\definecolor{cvprblue}{rgb}{0.21,0.49,0.74}
\usepackage[pagebackref,breaklinks,colorlinks,allcolors=cvprblue]{hyperref}

\def\paperID{44} 
\def\confName{CVPR}
\def\confYear{2026}

\title{GridVQA-X: A Framework for Evaluating Multimodal Explainability Methods}

\author{
Sujay Belsare$^{1,}$\thanks{Equal Contribution} \\
{\tt\small sujay.belsare@students.iiit.ac.in}
\and
Sudarshan Nikhil$^{1,}$\footnotemark[1] \\
{\tt\small sudarshan.nikhil@students.iiit.ac.in}
\and
Sushant Kumar$^{1}$\\
{\tt\small sushant.k@research.iiit.ac.in}
\and
Ponnurangam Kumaraguru$^{1}$\\
{\tt\small pk.guru@iiit.ac.in}
\and
Chirag Agarwal$^{2}$\\
\vspace{-2in}
\and
{\small
$^{1}$IIIT Hyderabad, India \quad
$^{2}$University of Virginia, USA
}
}

\begin{document}
\twocolumn[{%
\renewcommand\twocolumn[1][]{#1}%
\maketitle
\begin{center}
        \vspace{-0.25in}
        \centering
        \includegraphics[width=0.99\textwidth]{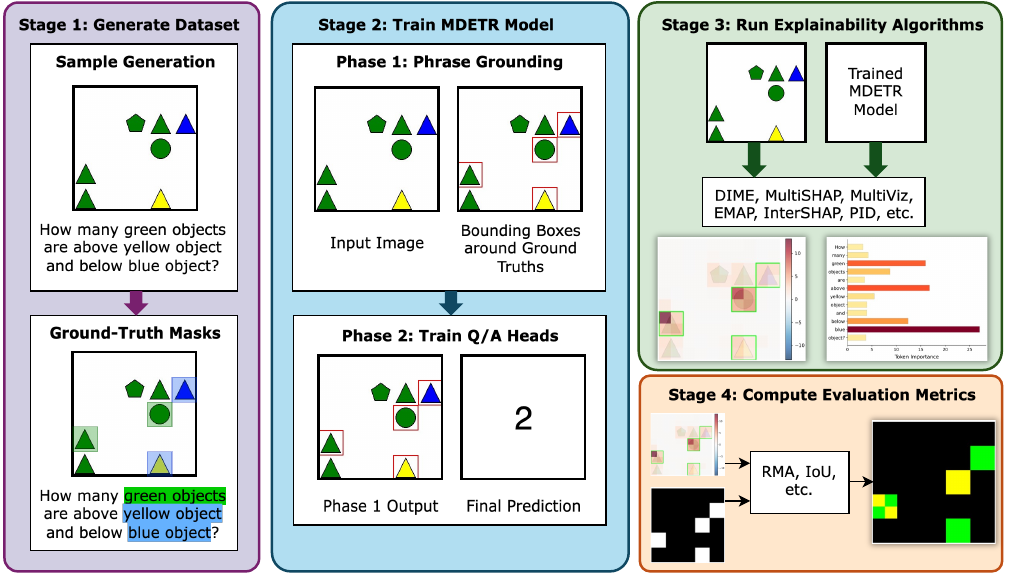}
        \captionof{figure}{\looseness=-1\textbf{\method framework:} Four-stage pipeline: (1) dataset generation with ground-truth phrase masks, (2) two-phase model training for phrase grounding and QA, (3) attribution generation using post-hoc explainers, and (4) evaluation by comparison with ground-truth masks.}
        \label{fig:teaser}
\end{center}%
}]
\begin{abstract}
With the increasing development of Vision-Language Models, it becomes imperative that their predictions are readily explainable to relevant stakeholders. However, the field of explainability has not kept pace with the multimodal surge. While recent Multimodal Explainable AI (MxAI) methods generate explanations to attribute the interaction between different modalities, current evaluation protocols lack the ground truth required to distinguish between true cross-modal reasoning (\eg spatial composition) and shallow cross-modal shortcuts (\eg Bag-of-Words attribute matching). It remains unknown whether MxAI methods faithfully capture synergistic interactions or merely hallucinate reasoning on models acting as simple feature detectors. In this paper, we introduce \method, the first diagnostic framework specifically designed to evaluate cross-modal explainability. Unlike natural datasets, \method leverages a closed-world synthesis logic to generate unique, mathematically guaranteed explanations. We utilize this controlled environment to train paired ground-truth models on identical architectures: $M_{\text{pure}}$, which learns robust spatial-relational reasoning and $M_{\text{spur}}$, which is structurally forced to rely on cross-modal shortcuts. This behavioral divergence creates a rigorous testbed: \textit{a faithful explainer must report distinct reasoning pathways for each model}.
Our findings reveal that widely used methods fail to distinguish between models relying on genuine spatial-relational reasoning and those exploiting cross-modal shortcuts, highlighting a critical gap in capturing true cross-modal synergy and misrepresenting how multimodal models actually make decisions. Our code is available at \href{https://github.com/deephelms28/GridVQA-X}{link}.
\end{abstract}
    
\section{Introduction}
\label{sec:intro}
\looseness=-1 The integration of visual and textual modalities has driven unprecedented performance in Large Vision-Language Models (LVLMs). However, this success is often achieved through deeply entangled representations, rendering the models highly opaque \cite{rodis2024multimodalexplainableartificialintelligence}. To facilitate the adoption of these predictive advancements in high-stakes applications like healthcare \cite{li2023llavamedtraininglargelanguageandvision, zhang2025biomedclipmultimodalbiomedicalfoundation}, post-hoc multimodal explainable AI (MxAI) methods have emerged~\cite{lyu2022dimefinegrainedinterpretationsmultimodal, hessel2020doesmultimodalmodellearn, liang2023multivizvisualizingunderstandingmultimodal, Wenderoth_2025, wang2026multishapshapleybasedframeworkexplaining, goldshmidt2025attentionpleasepixelshapreveals, liang2023quantifyingmodelingmultimodal}. These methods claim to attribute model decisions to input modality features by uncovering complex cross-modal interactions. Among these MxAI methods, those agnostic to the model architecture and downstream tasks hold the greatest promise for applicability across varied domains.

\looseness=-1 While existing post-hoc, task- and model-agnostic MxAI methods explain the decisions of multimodal models, there is little to no guarantee whether they faithfully reflect the model's true decision-making process.
Furthermore, although many explainability benchmarks have been recently introduced~\cite{Nauta_2023, ijcai2023p747} utilizing either ground-truth labels~\cite{Salewski_2022, Arras_2022, agarwal2024openxaitransparentevaluationmodel} or proxy measures~\cite{li2023mathcalm} across various datasets and models, their evaluation remains largely limited to unimodal methods.
These unimodal evaluation protocols cannot be trusted in multimodal domains where features interact across different semantic spaces \cite{zhang2025crossmodalinformationflowmultimodal}. Consequently, the rigorous assessment of these post-hoc MxAI methods \textbf{remains completely unexplored}, representing a critical gap in explainability research that we seek to fill.

\looseness=-1 Unlike their unimodal counterparts, MxAI methods claim to capture the complex interactions that arise when modalities fuse \cite{lyu2022dimefinegrainedinterpretationsmultimodal, li2025multimodalrationalesexplainablevisual}. Therefore, a gold-standard test for these methods requires the absolute ground-truth for these interactions. Diverse real-world multimodal tasks (\eg VQA, visual reasoning) have noisy feature distributions, offering no causally defined ground-truth attributes per modality. Furthermore, human-annotated features can easily be spuriously correlated with non-ground-truth visual elements, invalidating any faithfulness evaluation if the underlying model is merely exploiting that shortcut \cite{yuksekgonul2023when, chi2025chimeradiagnosingshortcutlearning}. Due to the presence of another modality, shortcuts can also sit entirely within the cross-modal space (\eg shallow Bag-of-Words matching), making it exceptionally difficult to identify the true generative process of the model's prediction.

To address these challenges, we propose \method: the first diagnostic framework explicitly designed to objectively evaluate the faithfulness of post-hoc cross-modal explainers, which we utilize to benchmark recent state-of-the-art methods. In particular, our contributions are threefold:
\begin{enumerate}
    \item \looseness=-1 We introduce two synthetic datasets, $\mathcal{D}_{\text{pure}}$ and $\mathcal{D}_{\text{spur}}$, acting as a controlled testbed for MxAI evaluation. These datasets feature mathematically guaranteed unique ground-truth features, a taxonomy leveraging spatial reasoning to isolate the diverse degree of real-world interactions, and the provable absence ($\mathcal{D}_{\text{pure}}$) or presence ($\mathcal{D}_{\text{spur}}$) of known spurious correlations. We also release the data generation code for reproducible and customization scaling.
    \item \looseness=-1 We release two reference models ($M_{\text{pure}}$ and $M_{\text{spur}}$) trained via explanation-guided dynamics to achieve near-perfect accuracy on their respective datasets. By verifiably knowing the models' underlying reasoning mechanisms, \ie true spatial reasoning vs. cross-modal shortcut, we can evaluate explainers with zero ambiguity \textit{w.r.t.} the expected ground-truth attributions.
    \item We adapt existing MxAI metrics to multimodal domain and introduce novel metric to assess the cross model interaction scalar on compositional complexity. By doing this we provide the first rigorous benchmark for evaluating post-hoc cross-modal explainers analyzing the failure modes.    
\end{enumerate}
\section{Related Work}
\label{sec:related}
Our work lies at the intersection of explainability methods and multimodal models. Below, we briefly describe them and detail them further in Appendix~\ref{app:related}.

\looseness=-1\noindent\xhdr{Multimodal Datasets} The transition toward analyzing modern visual reasoning systems began with datasets like CLEVR \cite{johnson2016clevrdiagnosticdatasetcompositional}, which provide programmatically generated 3D scenes and questions to evaluate whether a system genuinely reasons rather than exploiting the aforementioned spurious correlations. To extend this rigorous evaluation to natural images, the GQA \cite{hudson2019gqanewdatasetrealworld} dataset was introduced, leveraging real-world scene graphs to generate compositional questions while strictly balancing answer distributions to neutralize statistical world priors. The CLEVR-XAI \cite{Arras_2022} dataset supplies question-conditioned and pixel-level ground truth masks, allowing us to verify whether an AI's visual explanations truly align with its underlying logic, albeit only for deep neural networks.

\looseness=-1\noindent\textbf{MxAI Methods:} Recent papers highlight the growing need for interpretability techniques to understand multimodal LLMs~\cite{agarwal2025rethinkingexplainabilityeramultimodal}. Existing research \cite{dang2024explainableinterpretablemultimodallarge, sun2024reviewmultimodalexplainableartificial} can be broadly classified into three methodological categories: i) comprising game-theoretic, Shapley-based attribution methods~\cite{Wenderoth_2025, goldshmidt2025attentionpleasepixelshapreveals, wang2026multishapshapleybasedframeworkexplaining}; ii) focusing on internal representation and visualization techniques~\cite{liang2023multivizvisualizingunderstandingmultimodal}, where they scaffold interpretability into distinct stages to help users simulate predictions, understand feature representations, and debug errors; and iii) consisting of perturbation and projection-based approaches~\cite{hessel2020doesmultimodalmodellearn, lyu2022dimefinegrainedinterpretationsmultimodal}.

\begin{table}
\centering
\scriptsize
\setlength{\tabcolsep}{2pt}
\renewcommand{\arraystretch}{1}
\caption{Comparison of GridVQA-X with existing benchmarks}\vspace{-0.15in}
\begin{tabular}{lcccc}
\toprule
\textbf{Criteria} & \method & Clevr-XAI~\cite{Arras_2022} & OpenXAI~\cite{agarwal2024openxaitransparentevaluationmodel} & LATEC~\cite{klein2025navigatingmazeexplainableai} \\
\midrule

Multimodal & \textcolor{green}{\ding{51}} & \textcolor{red}{\ding{55}} & \textcolor{red}{\ding{55}} & \textcolor{red}{\ding{55}} \\

Unique Explanations & \textcolor{green}{\ding{51}} & \textcolor{red}{\ding{55}} & \textcolor{red}{\ding{55}} & \textcolor{red}{\ding{55}} \\

Process Identifiability & \textcolor{green}{\ding{51}} & \textcolor{red}{\ding{55}} & \textcolor{red}{\ding{55}} & \textcolor{red}{\ding{55}} \\

Controlled Environment & \textcolor{green}{\ding{51}} & \textcolor{green}{\ding{51}} & \textcolor{green}{\ding{51}} & \textcolor{red}{\ding{55}} \\

\bottomrule
\end{tabular}
\end{table}

\noindent\textbf{Comparison with Existing Benchmarks:}
While existing explainability benchmarks have made significant strides, they remain structurally insufficient for evaluating cross-modal synergy. Frameworks like OpenXAI \cite{agarwal2024openxaitransparentevaluationmodel} focus exclusively on unimodal data, offering no mechanisms to assess interactions across vastly different semantic spaces. Conversely, existing vision-language XAI benchmarks such as CLEVR-XAI \cite{Arras_2022} and LATEC \cite{klein2025navigatingmazeexplainableai} lack the strict identifiability of the model's generative process; because the true reasoning pathway of the underlying black-box model remains opaque, it is impossible to definitively know if the model utilized genuine compositional reasoning or a shallow cross-modal shortcut. Furthermore, these environments often lack mathematically unique ground-truth explanations, relying instead on proxy metrics or subjective human annotations. \textit{\method overcomes these limitations by providing a strictly controlled multimodal environment. By mathematically guaranteeing unique ground-truth features and evaluating against paired models with verifiably identified reasoning pathways, \method enables an objective, zero-ambiguity assessment of multimodal explainers.}

\section{The GridVQA Dataset}
\label{sec:dataset}

\begin{figure*}[t]
    \centering
    \includegraphics[width=0.9\textwidth]{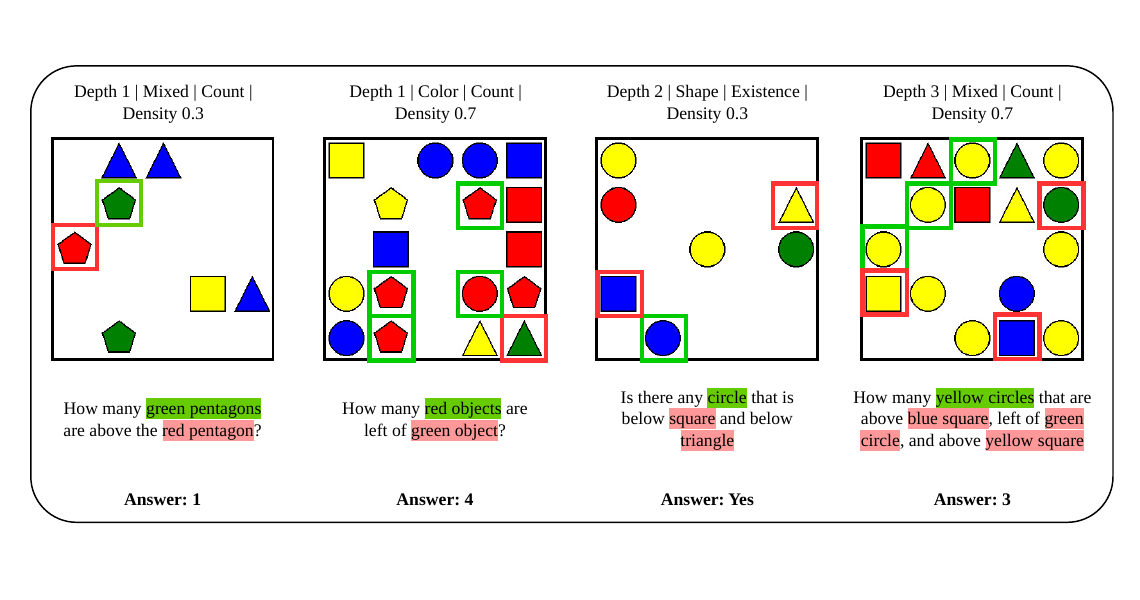} 
    \vspace{-0.35in}
    \caption{\textbf{The GridVQA Taxonomy.} Examples illustrating the dataset parameterization axes. The \textit{Density} axis ($d_{0.3}$ vs. $d_{0.7}$) models real-world background noise.  The \textit{Depth} axis, scaling relational complexity from single-hop to multi-hop spatial compositions. The textual queries highlight targets and anchors, demonstrating how \textit{QType} selectively omits (like in Shape-Only, Color-Only) or ensures (like in Mixed) visual confounders to force robust reasoning.}
    \label{fig:taxonomy}
\end{figure*}

To facilitate the rigorous assessment of multimodal explainability methods, we utilize spatial reasoning as a playground for controlled evaluation. In any real-world scenario, a visual scene is fundamentally a composition of atomic objects possessing intrinsic features (\eg color, shape) and extrinsic semantic relationships (\eg spatial positioning). GridVQA distills this complexity into a controlled, fully observable abstraction. A standard GridVQA sample consists of a $S \times S$ visual grid populated by geometric objects, paired with a multi-hop spatial reasoning question. Every query defines a \textbf{Target} (the object(s) the model must find/count) and one or more \textbf{Anchors} (the reference objects used to locate the target) connected by strict spatial directional tokens. Crucially, GridVQA's generation axes are explicitly designed to mimic the dynamics of real-world multimodal tasks, acting as a necessary stress test for MxAI algorithms before they can be trusted in high-stakes scenarios.

\subsection{Design and Taxonomy}
\label{sec:taxonomy}
Every sample in the GridVQA universe (see Fig.~\ref{fig:taxonomy} for some examples) is parameterized by a strict 4-tuple $\mathcal{T} = (D, Q, F, \rho)$:
\begin{itemize}
    \item \textbf{Depth $D \in \{1, 2, 3\}$} controls the number of anchor objects in the image, and thus the complexity of the question. For example, a Depth 1 query (\textit{``...left of the blue circle?''}) contains one anchor. A Depth 2 query (\textit{``...left of the blue circle and above the red square?''}) requires intersecting relations from two anchors and so on.
    \item \textbf{QType} $Q \in \{A, SO, CO, M, CMP\}$ specifies the type of question being asked. See Table~\ref{tab:qtype} for more details.
    \item \textbf{Form $F \in \{0, 1\}$} alters the objective without changing the underlying causal graph. Form $0$ poses a counting task (\textit{``\textbf{How many} red squares are...''}), while Form $1$ poses an existence task (\textit{``\textbf{Is there} any red square...''}). This ensures explanations are tied to the true causal reasoning of the scene, rather than overfitting to the statistical format of the task header.
    \item \textbf{Density $\rho \in \{d_{0.3}, d_{0.7}\}$} controls the total number of objects in the grid. Sparse grids allow us to test spatial localization fidelity, whereas dense grids test the explainer's ability to cut through the distractor objects.
\end{itemize}

\subsection{Provably Unique Ground Truth Features}
\label{sec:formalism}

To mathematically benchmark MxAI methods and prove the absence of shortcuts, we define the limits of our fully observable universe. Let the visual scene $\mathcal{V}$ be defined as a set of $N$ objects $\{o_1, \dots, o_N\}$. Each object is a tuple of independent atomic features $o_i = (c_i, s_i, p_i)$ representing color, shape, and position. The multimodal query $\mathcal{Q}$ defines attribute constraints for target objects $\mathcal{T}$, anchor objects $\mathcal{A}$, and the spatial constraint $p_{\text{rel}}$ connecting them.

\noindent\textbf{Causal Intervention:} Given the fully observable state of $\mathcal{V}$, we construct a causal framework to isolate the ground-truth explanation. Let $\text{Dist} = \mathcal{V} \setminus (\mathcal{A} \cup \mathcal{T})$ be the set of all remaining distractor objects. Applying Pearl's $do$-calculus \cite{tucci2013introductionjudeapearlsdocalculus}, we define an intervention $do(o_d \to o_d')$ that alters any distractor $o_d \in Dist_{adv}$. Because $\mathcal{Q}$ is a strict logical mapping exclusively over $\mathcal{T}$ and $\mathcal{A}$, it is mathematically guaranteed that:
\begin{equation}
    P(y \mid do(o_d \to o_d')) = P(y)
\end{equation}
The causal effect of the set $Dist_{adv}$ on the output $y$ is strictly zero. Therefore, we mathematically prove that $\mathcal{A} \cup \mathcal{T}$ constitutes the unique, ground-truth causal explanation for $\mathcal{Q}$. Any explainer attributing relevance to elements in $Dist_{adv}$ is \textbf{verifiably unfaithful}.

\begin{table*}[t]
\centering
\scriptsize
\setlength{\tabcolsep}{4pt}
\renewcommand{\arraystretch}{1.2}
\caption{\method Question Type (QType) definitions and motivation}
\label{tab:qtype}\vspace{-0.12in}
\begin{tabular}{lp{0.81\textwidth}}
\toprule
\textbf{QType} & \textbf{Description and Motivation} \\
\midrule
Attribute-Only (A)
& Non-spatial queries regarding existence or counting of objects identified by both shape and color. 
Baseline to verify that the explainer can reveal objects before the added complexity of spatial reasoning. 
\newline \textit{Example:} Is there any \textbf{red circle}? \\\hline
Shape-Only (SO)
& Spatial relational queries where both target and anchor are identified by shape only. Isolates the explainer's ability to ground directional reasoning independently of color features.
\newline \textit{Example:} Is there any \textbf{circle} that is \textit{left of} \textbf{pentagon}? \\\hline
Color-Only (CO)
& Spatial relational queries where both target and anchor are identified solely by color. Isolates the explainer's ability to ground directional reasoning independently of shape features. 
\newline \textit{Example:} Is there any \textbf{red object} that is \textit{left of} \textbf{blue object}? \\\hline
Mixed (M)
& Spatial queries requiring the model to bind shape and color for both target and anchor. A faithful explainer must show the model utilizing the anchor's attributes and spatial relation to locate the target; failure here reveals reliance on shallow attribute-based shortcuts. 
\newline \textit{Example:} Is there any \textbf{red circle} that is \textit{left of} \textbf{blue pentagon}? \\\hline
Comparison (CMP)
& Logical operations comparing counts of different objects with both shape and color specified. Evaluates multi-step logical reasoning and grounding dependencies across the grid. 
\newline \textit{Example:} Are there more \textbf{red circles} than \textbf{blue pentagons}? \\
\bottomrule
\end{tabular}
\vspace{0.05in}
\end{table*}

\begin{figure*}[!t]
    \centering
    \includegraphics[width=0.95\textwidth]{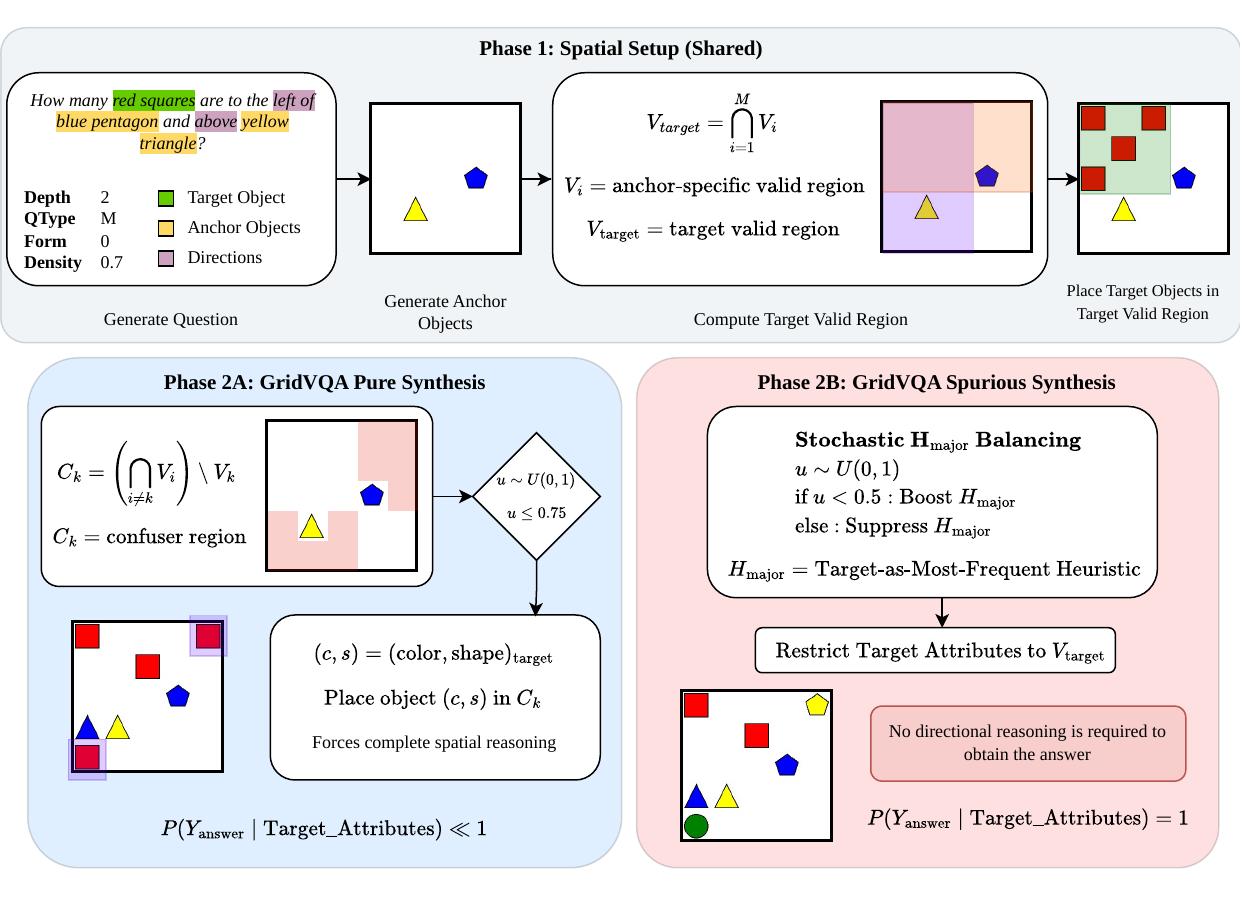}
    \caption{\textbf{The Dataset Generation Process.} \textbf{Phase 1:} Anchors are instantiated and intersected to determine the valid target bounding region ($V_{\text{target}}$). \textbf{Phase 2A ($\mathcal{D}_{\text{pure}}$):} To destroy spatial shortcuts, the generator computes adversarial confuser regions ($C_k$) and probabilistically overloads them with target-attribute distractors, forcing strict multi-hop evaluation. \textbf{Phase 2B ($\mathcal{D}_{\text{spur}}$):} The generator nullifies visual majority heuristics but actively restricts the attribute sampling pool, ensuring no target-attribute objects leak outside $V_{\text{target}}$, embedding the Case-1 trap.}
    \label{fig:generation_architecture}
\end{figure*}

\subsection{Spurious Correlation Elimination and Dataset Divergence}
\label{sec:spurious}

To systematically evaluate explainability algorithms, we must definitively know the ground-truth reasoning of the black-box models. If an explainer highlights the token ``\textit{left of}'', is it faithfully revealing the model's logic, or hallucinating a rationale onto a model utilizing a shortcut? We solve this by defining the bounds of the heuristic hypothesis space. Rather than arbitrarily selecting heuristics to evaluate, we systematically enumerate the shortcut hypothesis space by analyzing the proper subsets of the true causal feature set ($\Phi_h \subset \Phi_{\text{true}}$) within our closed-world formalism. This taxonomy yields four fundamental families of spurious correlations: Answer Priors (Case-0), Bag-of-Words Alignment (Case-1), Visual Feature Dominance (Case-2), and Logical Decomposition (Case-3). Crucially, because these cases represent the atomic failures of multimodal grounding, any higher-degree heuristic (\eg a complex shortcut relying on multiple missing constraints simultaneously) is inherently neutralized by the same structural guarantees that destroy its atomic components. 

\begin{definition}[Spurious Correlation]
Let the true causal dependency require the feature set $\Phi_{\text{true}} \subseteq \mathcal{V} \times \mathcal{Q}$.$\mathcal{Y}$ is the ground-truth. We define a heuristic $h$ as a function relying on a proper subset of features $\Phi_{h} \subset \Phi_{true}$. A spurious correlation exists if the Mutual Information (MI) exceeds a threshold $\epsilon$ (random chance): $MI(h(\Phi_h), \mathcal{Y}) > \epsilon$
\end{definition}

\looseness=-1 We provide two parallel datasets to control these heuristics. $\mathcal{D}_\text{pure}$ systematically removes all enumerated heuristic cases, guaranteeing the model learns the true causal graph. Simultaneously, $\mathcal{D}_\text{spur}$ eliminates unimodal biases (\eg visual dominance of the target attribute object) but intentionally preserves the most pervasive cross-modal shortcut. The mathematical elimination of these heuristics in $\mathcal{D}_\text{pure}$, and the precise statistical injection of the Case-1 (see below) shortcut in $\mathcal{D}_\text{spur}$, have been empirically confirmed across all dataset splits. Comprehensive generative reports detailing the predictive failure rates of these heuristics are provided in Appendix~\ref{app:dataset}.

For a given anchor object $o_i$ in a query $\mathcal{Q}$, we define its valid region $V_i$ as the set of all positions $p_j$ in the grid such that they satisfy the special constraint associated with that anchor in $\mathcal{Q}$ (\eg all positions with x-coordinate strictly less than that of the anchor for a ``\textit{left of}'' relation).\vspace{0.05in}

\xhdr{Case-1: The Spatial Shortcut (Bag-of-Words Alignment)}
The most pervasive heuristic in VQA is mapping semantic tokens directly to visual features while ignoring relationships (in our case, spatial relationships) \cite{yuksekgonul2023when} (\eg counting all ``\textit{red squares}'' instead of actually checking if they are ``\textit{left of the blue circle}'').  

In  $\mathcal{D}_\text{spur}$, the algorithm is programmed to make this shortcut perfectly predictive. By restricting the attribute sampling pool during noise injection (Fig.~\ref{fig:generation_architecture}), the generator places no target-matching objects outside the valid spatial region. Consequently, $P(Y_{\text{ans}} \mid \text{Target\_Attrs}) = 1.0$, embedding a statistical behavioral trap. Conversely, $\mathcal{D}_\text{pure}$ mathematically severs this correlation by establishing the following structural property:

\looseness=-1 \noindent \textit{Property 1 (Spatial Independence). } For every relational query $q = R(\text{Target\_}\allowbreak\text{Attrs}, \text{Anchor})$, the generative logic ensures the existence of an adversarial distractor set $Dist_{adv} \subset \mathcal{V}$ such that every object $d_{i} \in Dist_{\text{adv}}$ satisfies the target's visual attributes but explicitly violates the spatial constraint $R$. 

By populating invalid regions with adversarial distractors, the global count of target-attribute objects inherently exceeds the true relational count $Y_{\text{ans}}$. The probability of the shortcut yielding the correct answer rapidly decays to zero ($P(Y_{\text{ans}} \mid \text{Target\_Attrs}) \ll 1$), mathematically forcing the model to process spatial geometry.\vspace{0.05in}

\xhdr{Case-2: Visual Feature Dominance} To prevent the model from blindly predicting based on the most frequent object type ($H_{\text{maj}}$), the generator must logically decorrelate the target from the visual majority:

\noindent \textit{Property 2 (Orthogonality of Frequency).} Let $N(a)$ be the count of objects with attribute $a$. The generative logic actively enforces target-distractor independence such that $P(Y_{gt} = N(a) \mid a = \arg\max_{a'} N(a')) \ll 1$.

To satisfy this, the $\mathcal{D}_\text{spur}$ generator utilizes a stochastic branching mechanism (Fig.~\ref{fig:generation_architecture}), dynamically boosting or suppressing the target to nullify the visual majority dominance. The details of the algorithm are provided in Appendix~\ref{app:proof}. Note that the elimination of the Case 1 heuristic for $\mathcal{D}_\text{pure}$ also eliminates this heuristic, as if the target attributes match $H_{\text{maj}}$, then the presence of adversarial distractors makes the model always overcount the number of spatially constrained answers, making this heuristic fail.\vspace{0.05in}

\xhdr{Case-3: Logical Decomposition (Partial Logic)}
This heuristic occurs when a model receives a multi-hop query but evaluates only a subset of the constraints. If a model drops constraint $k$, evaluating $\bigwedge_{i \neq k} C_i$, it learns to safely ignore anchor $k$ without penalty. To eliminate this in $\mathcal{D}_\text{pure}$, we generalize the robust intersection requirement:

\begin{theorem}[\textbf{Generalized Robust Intersection Guarantee}]
Let $V_i$ be the valid spatial region for anchor $i$, and $V_{-k} = \bigcap_{i \neq k} V_i$ be the relaxed region. To guarantee the failure of partial logic, the generative logic must ensure the confuser region $C_k$ is non-empty and populated with at least one adversarial object:

\begin{equation}
    C_k = V_{-k} \setminus V_k = \left( \bigcap_{i \neq k} V_i \right) \setminus V_k
\end{equation}
\end{theorem}

\looseness=-1 \noindent\textit{Proof Sketch.} Let $Y_{\text{true}}$ be the set of valid target objects located within $V_{\text{target}}$, where $V_{\text{target}} = V_{-k} \cap V_k$. A model employing the partial logic heuristic evaluates only the relaxed region $V_{-k}$, identifying a candidate set of objects $Y_{\text{partial}}$. Because $V_{\text{target}} \subseteq V_{-k}$, it follows that $Y_{\text{true}} \subseteq Y_{\text{partial}}$. The false positives detected by this heuristic lie strictly in the relative complement $C_k = V_{-k} \setminus V_k$. By algorithmically guaranteeing that $C_k$ contains at least one adversarial object matching the target's visual attributes, we ensure $Y_{\text{partial}} \supset Y_{\text{true}}$. Consequently, the heuristic strictly overcounts (or falsely detects) targets, yielding an incorrect prediction and incurring a training penalty.

\noindent\textit{Implications.} If Theorem 1 holds, a model dropping anchor $k$ will erroneously detect adversarial objects within $C_k$, incurring a training loss. The $\mathcal{D}_\text{pure}$ generator actively computes $C_k$ for every anchor and routes adversarial spatial distractors into them (see Fig.~\ref{fig:generation_architecture}: the circled distractors in 2A), neutralizing the heuristic. See Appendix~\ref{app:proof} for the detailed proof.

\subsection{Dataset Generation Architecture}
\label{sec:generation}
\looseness=-1 To instantiate the formalisms described above and generate the $\mathcal{D}_{\text{pure}}$ and $\mathcal{D}_{\text{spur}}$ splits, we utilize a procedural generation logic. The algorithmic divergence between the two splits occurs strictly during the distractor injection phase.

As detailed in Algorithm~\ref{alg:generation}, Phase 1 (Lines 3-6) handles the semantic instantiation of the query, anchors, and true targets uniformly for both environments. The divergence occurs in Phase 2. The \textbf{Pure} module algorithmically computes the confuser regions ($C_k$) for every anchor and actively injects adversarial target-attribute distractors into these specific spaces (Lines 10-14). This explicit, deterministic placement mathematically guarantees the Robust Intersection regime (Theorem 1) without relying on inefficient rejection sampling. Conversely, the \textbf{Spurious} module intentionally seeds distractors across the grid until the total count of objects matching the target's attributes strictly equals the ground-truth answer $Y_{\text{ans}}$ (Line 18), actively embedding the Case-1 Bag-of-Words correlation directly into the visual scene.
\begin{algorithm}[t]
\scriptsize
\caption{GridVQA Procedural Divergence Logic}\label{alg:generation}
\begin{algorithmic}[1]
\Procedure{GenerateScene}{$\mathcal{T}, \text{Mode} \in \{\text{Pure}, \text{Spurious}\}$}
    \State \Comment{\textbf{Phase 1: Shared Spatial Setup}}
    \State $\mathcal{Q}, \text{Attrs} \gets \text{SampleQuestionTemplate}(\mathcal{T})$
    \State $\mathcal{A}, Pos_{\text{anc}} \gets \text{PlaceAnchors}(\text{Attrs.Anchor})$
    \State $\mathcal{T}_{\text{tgt}} \gets \text{PlaceTargets}(\text{Attrs.Target}, Pos_{\text{anc}}, \text{Attrs.Direction})$
    \State $\mathcal{D} \gets \emptyset$
    \State $N_{\text{dist}} \gets \text{CalculateRequiredNoise}(\mathcal{T}.\rho) - |\mathcal{A} \cup \mathcal{T}_{\text{tgt}}|$
    
    \State \Comment{\textbf{Phase 2: Divergent Distractor Injection}}
    \If{$\text{Mode} == \text{Pure}$}
        \For{$k = 1 \dots |\mathcal{A}|$} \Comment{\textbf{Enforce Theorem 1}}
            \State $C_k \gets \left( \bigcap_{i \neq k} V_i \right) \setminus V_k$ \Comment{Compute Confuser Region}
            \If{$C_k \neq \emptyset$}
                \State $\mathcal{D} \gets \mathcal{D} \cup \text{Inject}(\text{Sample}(C_k), \text{Attrs.Target})$
            \EndIf
        \EndFor
        \State $\mathcal{D} \gets \mathcal{D} \cup \text{StochasticConfuserOverload}(C_k, N_{\text{dist}} - |\mathcal{D}|)$
    \ElsIf{$\text{Mode} == \text{Spurious}$}
        \State $Y_{\text{ans}} \gets |\mathcal{T}_{\text{tgt}}|$
        \State $\mathcal{D} \gets \text{InjectShortcutObjects}(\text{Attrs.Target}, Y_{\text{ans}})$ \Comment{Forces Case-1 Shortcut}
        \State $\mathcal{D} \gets \mathcal{D} \cup \text{SampleRandomObjects}(N_{\text{dist}} - |\mathcal{D}|)$
    \EndIf
    \State \textbf{return} $\mathcal{A} \cup \mathcal{T}_{\text{tgt}} \cup \mathcal{D}, \mathcal{Q}, \text{Masks}(\mathcal{A}, \mathcal{T}_{\text{tgt}}, \mathcal{D})$
\EndProcedure
\end{algorithmic}
\end{algorithm}

\section{Model Training and Behavioral Dynamics}
\label{sec:models}

We employ a unified transformer-based architecture, specifically MDETR \cite{kamath2021mdetrmodulateddetection}, as our controlled test subject. By varying the training environments ($\mathcal{D}_{\text{pure}}$ vs.\ $\mathcal{D}_{\text{spur}}$) while keeping the architecture and loss functions identical, we isolate the distinct reasoning behaviors required for our diagnostic testbed.

\subsection{Training Dynamics and Explanation-Guided Learning}

MDETR is trained in two phases, which act as explanation-guided training \cite{ross2017rightrightreasonstraining}. Phase 1 (visual grounding) enforces visual–text alignment by penalizing unpredicted reference bounding boxes using a combination of L1 and generalized IoU losses: $\mathcal{L}_{\text{ground}} = \lambda_{L1} \mathcal{L}_{1} + \lambda_{\text{giou}} \mathcal{L}_{\text{giou}}$. Phase 2 (QA) then processes these grounded representations to predict the logical answer $Y_{\text{ans}}$. The computational difficulty of multi-hop spatial intersections frequently causes standard models to locally collapse into Answer Prior Bias (Case-0). To prevent this, Phase 2 employs a dynamically weighted cross-entropy loss: $\mathcal{L}_{\text{QA}} = -w_c \log P(\hat{Y}_{\text{ans}} = y_c)$, where $w_c$ is inversely proportional to the batch class frequency.

\subsection{Verifying the Shortcut}
To ensure our diagnostic testbed is valid and that the dataset divergence successfully imprinted the desired behaviors, we perform a rigorous cross-evaluation of the trained models. When evaluated on its training distribution, $M_{\text{spur}}$ achieves perfect accuracy ($1.000$). However, when evaluated on $\mathcal{D}_{\text{pure}}$, which introduces adversarial distractors that break the spatial shortcut, its performance drops to 49\%.

It's accuracy on multi-hop relational queries drops to $8\%$ on Depth-2 Mixed queries and $14\%$ on Depth-3 Mixed queries. Notably, it's performance on non-relational Attribute-Only queries remains at $100\%$, confirming that the model's failure is strictly isolated to its inability to process spatial composition. This empirical analysis confirms our theoretical guarantees: $M_{\text{spur}}$ operates entirely on a unimodal, Bag-of-Words shortcut (Case-1). Conversely, $M_{\text{pure}}$ achieves robust accuracy across the pure splits, empirically proving it has successfully internalized the true causal spatial-relational synergy.

\section{The MxAI Evaluation}
\label{sec:framework}

\looseness=-1 With the controlled models established, we define a comprehensive evaluation pipeline tailored to the formats of modern MxAI algorithms and address the following key questions: \textbf{(RQ1)} Can MxAI methods diagnose shortcut learning in multimodal models?
\textbf{(RQ2)} Do local explainers capture true cross-modal synergy or merely exploit visual volume?
\textbf{(RQ3)} Do global MxAI methods scale with compositional complexity?

\subsection{Experimental Setup}
\looseness=-1 Here, we outline the multimodal explainability methods alongwith the metrics and protocol adopted for their evaluation.

\xhdr{MxAI Taxonomy and Methods} We divide existing MxAI methods into two benchmark categories: \textit{Group A (Local-Level),} where we use methods \dime~\cite{lyu2022dimefinegrainedinterpretationsmultimodal}, \multishap~\cite{wang2026multishapshapleybasedframeworkexplaining}, and \multiviz-gradient~\cite{liang2023multivizvisualizingunderstandingmultimodal} that yield a cross-attribution heatmap, and \textit{Group B (Global-Level)}, where we use methods \emap~\cite{hessel2020doesmultimodalmodellearn} and InterSHAP~\cite{Wenderoth_2025} that output a global synergy scalar, $S_{\text{score}}$. See Appendix~\ref{app:related} for details about the explainability algorithms. More specificaly, these methods evaluated by our framework include - i) \dime: It disentangles model predictions into unimodal contributions and multimodal interactions and claims to enable fine-grained, architecture-agnostic analysis of multimodal behavior; ii) \multishap: A model-agnostic framework leveraging the Shapley Interaction Index to attribute predictions to synergistic and suppressive pairwise interactions between fine-grained cross-modal elements; and iii) \multiviz: It scaffolds interpretability into four stages: unimodal importance, cross-modal interactions, multimodal representations, and prediction composition, to facilitate error analysis and model debugging.\vspace{0.05in}

To further analyze the model behavior at a global level, we utilize three additional state-of-the-art methods: \textit{i) \emap:} It identifies the minimal adversarial perturbation required to alter a model's classification, effectively bridging feature-weighting paradigms with counterfactual reasoning; \textit{ii) InterSHAP:} This method quantifies cross-modal interactions using the Shapley interaction index to precisely separate individual modality contributions from synergistic effects across multiple data sources; and \textit{iii) PID:} An information-theoretic framework that decomposes multimodal interactions into independent, redundant, and synergistic components using a gradient-based Gaussian optimization approach.\vspace{0.05in}

\xhdr{Evaluation Metrics} We adapt established metrics to the multimodal domain, utilizing the $M_{\text{pure}}$ and $M_{\text{spur}}$ models as our ground truth for expected behavior. In particular, along with the standard IoU with otsu binarization, we use \textit{Relevance Mass Accuracy}. For a given mask:
\begin{equation}
\text{RMA}(I_{\text{map}}, \text{Mask}) =
\frac{\sum_{(x,y) \in \text{Mask}} |I_{\text{map}}(x,y)|}
{\sum |I_{\text{map}}|}
\end{equation}
\looseness=-1 A faithful MxAI method must assign near-zero relevance mass on $\text{Mask}_{\text{Distractor}}$ for $M_{\text{pure}}$, but high relevance for $M_{\text{spur}}$ on the same set. Finally, to prove an MxAI method captures true cross-modal synergy (rather than showing unimodal saliency), we introduce the additive fallacy check. For Group B methods, we track $S_{\text{score}}$ across the \textit{Depth} axis. True synergy capture requires $S_{\text{score}}$ to monotonically increase with depth for $M_{\text{pure}}$, while remaining invariant for $M_{\text{spur}}$.\vspace{0.05in}

\noindent\textbf{Experimental Setup:} To evaluate the fidelity and robustness of explainability methods, we analyze the above mentioned explainability algorithms. In particular, we evaluate explanations under three scenarios: \textbf{i) Pure Evaluation.} $\mathcal{M}_\text{pure}$ evaluated on $\mathcal{D}_\text{pure}$. This setting establishes a baseline by assessing whether the explainability algorithms correctly attribute importance to the intended anchors and targets; \textbf{ii) Spurious Evaluation.} $\mathcal{M}_\text{spur}$ evaluated on $\mathcal{D}_\text{spur}$. This scenario analyzes potential failure modes by examining whether the explanations reveal reduced importance assigned to directional cues while the model still produces the correct output; and \textbf{iii) Cross-Evaluation.} $\mathcal{M}_\text{spur}$ evaluated on $\mathcal{D}_\text{pure}$. This setting tests whether the explanations correctly reveal that the model assigns lower importance to the true targets and anchors when generating predictions.

\subsection{Results and Analysis}
\label{sec:results}

To systematically dissect the behavior of MxAI methods, we structure our analysis around the three core research questions (RQs) evaluating their diagnostic capacity, faithfulness grounding, and synergy estimation.

\begin{figure*}[!t]
    \centering
    \includegraphics[width=0.95\textwidth]{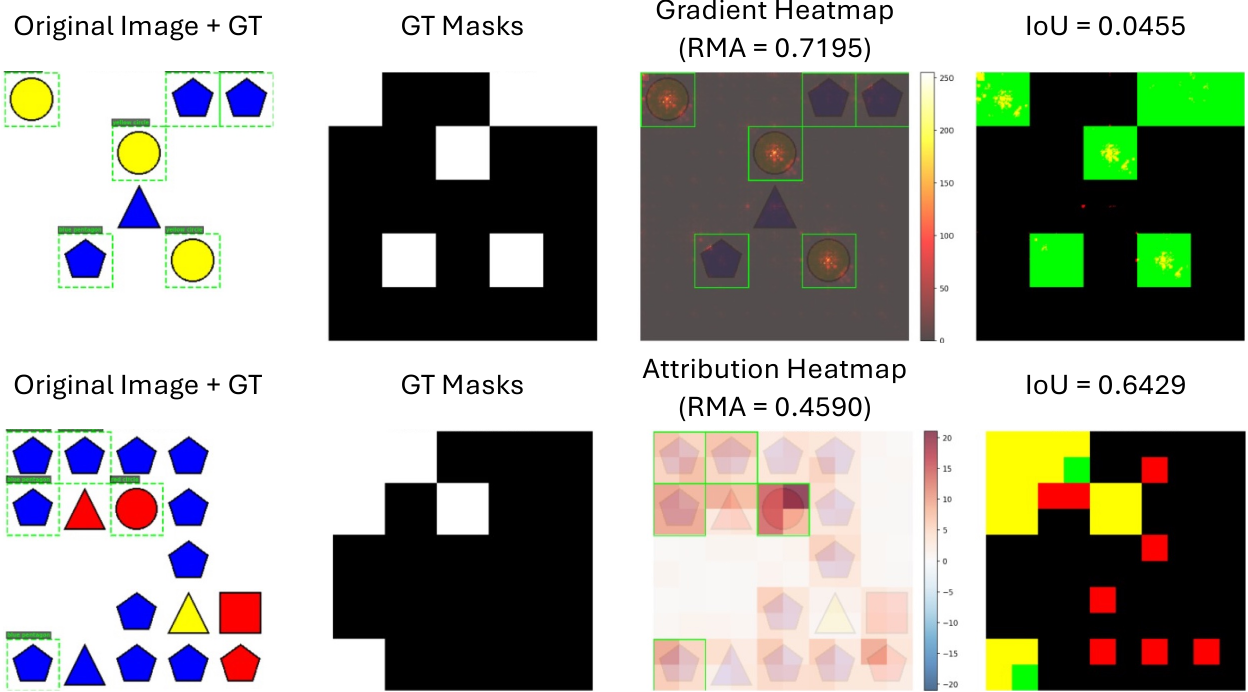}
\caption{\looseness=-1 Qualitative explanations for \multiviz~(top) and \multishap~(bottom). Q1: Are there more yellow circles than blue pentagons? (Ans: No) Q2: How many blue pentagons are left of the red circle? (Ans: 4). Progressively from left to right, we show the ground-truth mask (green), the method's prediction (red), and their overlap (yellow). We report RMA and IoU scores for different buckets across the three scenarios.}
\label{fig:qual_results}
\end{figure*}

\looseness=-1 The primary utility of \method is its ability to test whether an explainer can distinguish between true spatial reasoning ($M_{\text{pure}}$) and shallow cross-modal shortcuts ($M_{\text{spur}}$). A faithful explainer must produce vastly divergent attributions for these models, exposing $M_{\text{spur}}$'s reliance on Case-1 distractors. Overall, local methods fail this diagnostic test. \multiviz exhibits absolute model blindness, yielding statistically identical RMA ($\approx 0.44$) for both models despite their distinct causal pathways. \dime suffers from ``accidental faithfulness'': because its heatmaps are highly unconstrained and diffuse, they accidentally intersect with the true spatial target when evaluating $M_{\text{spur}}$, generating an illusion of correct reasoning that masks the shortcut. Conversely, \multishap yields \textit{higher} RMA for $M_{\text{spur}}$ (0.689) than for $M_{\text{pure}}$ (0.616). Because game-theoretic marginals align better with independent 1-to-1 feature detectors (the shortcut) than with entangled non-linear intersections, \multishap structurally prefers the spurious model. 
Global methods similarly struggle. \emap hallucinates synergy, attributing $\sim 60\%$ of the predictive power to cross-modal interaction on $M_{\text{spur}}$ even when the model catastrophically fails on $\mathcal{D}_{\text{pure}}$. Because masking either modality degrades the shallow shortcut, \emap misinterprets this dual-dependency as deep synergy.
\begin{figure*}[t]
    \centering
    \includegraphics[width=1.0\textwidth]{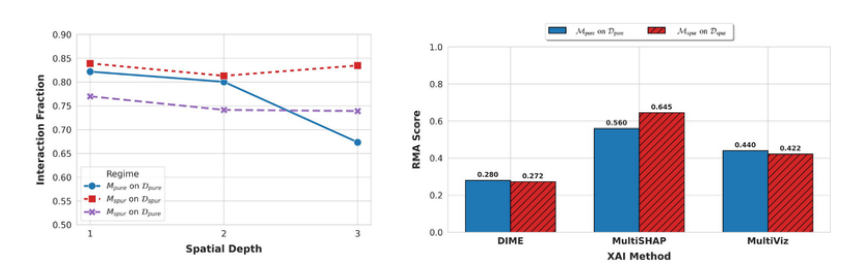}\vspace{-0.1in}
    \caption{\textbf{Quantitative Diagnostic Failures of MxAI Methods.} \textbf{(a) The Additive Fallacy:} \emap's global synergy metric incorrectly decays as spatial complexity (Depth) scales for the faithful model ($M_{\text{pure}}$), while hallucinating high synergy for the shortcut model ($M_{\text{spur}}$). \textbf{(b) RMA:} Comparison of average RMA scores demonstrating MultiSHAP's superior reasoning alignment on both faithful ($M_{\text{pure}}$) and shortcut-reliant ($M_{\text{spur}}$) models}\vspace{-0.1in}
    \vspace{-0.1in}
    \label{fig:qual_results_graph}
\end{figure*}

\begin{figure}
    \centering
    \includegraphics[width=0.5\textwidth]{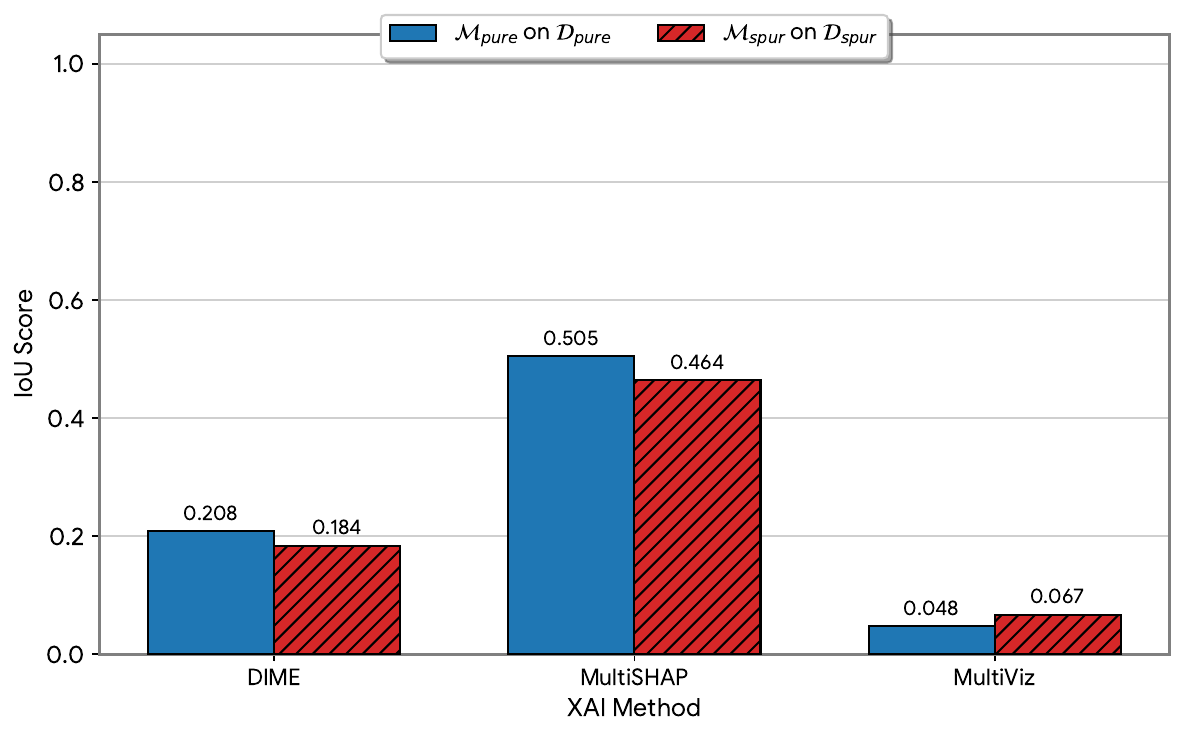}\vspace{-0.1in}
    \caption{\textbf{Comparison of the IoU scores across different Local Explainability algorithms and evaluation scenario.} Comparison of average IoU scores highlighting MultiSHAP's robust spatial grounding capabilities compared to the systemic attribution failures of DIME and MultiViz}\vspace{-0.1in}
    \vspace{-0.1in}
    \label{fig:iou_score}
\end{figure}
\looseness=-1 For \textbf{RQ2}, evaluating the micro-level explainers on $\mathcal{D}_{\text{pure}}$ reveals that they fundamentally fail to capture relational cross-modal constraints, defaulting instead to shallow feature detection or noise.
\textsc{Dime} acts as a ``spreader,'' heavily exploiting ground-truth mask volume. Its RMA artificially inflates in dense grids ($d_{0.7}$) simply due to target crowding, with no corresponding gain in spatial precision (IoU remains $\approx 0.20$). Furthermore, \dime exhibits no meaningful multimodal information gain over its unimodal baseline.
\textsc{\multishap} acts as a precise object detector (IoU $> 0.50$) but suffers from severe ``distractor leakage.'' It perfectly highlights all objects matching the target's unary attributes (color/shape), including adversarial distractors, proving an inability to apply the binary spatial constraint. This causes massive RMA penalties on Existence tasks (Form 1) compared to Counting tasks (Form 0), as false-positive distractors overwhelm the single true target.
\textsc{MultiViz} exhibits a total spatial grounding collapse. Despite moderate RMA, its IoU floors at $< 0.05$, indicating its visualizations highlight fragmented noise decoupled from actual object bounding boxes. Furthermore, InterSHAP retains a heavy bias towards mask volume, systematically assigning higher interaction values to Counting tasks ($0.921$) than equivalent Existence tasks ($0.742$).

\looseness=-1 To answer our \textbf{RQ3}, a faithful global synergy scalar must monotonically increase as the multi-hop relational complexity scales. \textsc{Emap} empirically fails this Additive Fallacy check. On Mixed (Form 0) queries in $\mathcal{D}_{\text{pure}}$, its estimated interaction fraction actually deteriorates as complexity scales, dropping from $0.821$ (Depth 1) $\to$ $0.673$ (Depth 3). By falsely indicating that tri-anchor queries require \textit{less} synergistic reasoning, \textsc{Emap} proves its estimator is misaligned with the causal graph. Similarly, \textsc{InterSHAP} echoes this failure; its interaction scores paradoxically plummet from $0.921$ (Depth 1) $\to$ $0.629$ (Depth 3) on Mixed queries, confirming that even game-theoretic interaction estimators fail the Additive Fallacy check and struggle to quantify scaling compositional complexity.

\section{Conclusion}
\label{sec:conclusion}
In this work, we introduced \method, the first diagnostic framework designed to rigorously evaluate multimodal explainability (MxAI) methods against mathematically verifiable ground-truth reasoning. Our exhaustive evaluation reveals a critical blindspot in current MxAI research: state-of-the-art explainers fundamentally fail to capture true cross-modal synergy. Local attribution methods either diffuse relevance arbitrarily to exploit mask volume or degrade into shallow object detectors, while global estimators hallucinate complex interactions on models utilizing simple cross-modal shortcuts. Ultimately, current MxAI methods are dangerously blind to the multimodal shortcuts, creating an illusion of interpretability that masks biased model behavior. By open-sourcing the \method generation engine and paired diagnostic models, we provide a definitive, zero-ambiguity testbed to shift the field away from superficial plausibility metrics and toward the development of explainers capable of verifiably diagnosing relational cross-modal grounding.

\section{Acknowledgment}
\label{sec:acknowledgement}
\looseness=-1 We would like to thank the PreCog Lab at IIIT-Hyderabad for their valuable guidance and discussions throughout this work. In particular, we thank  Vedanta S. P. and Debangan Mishra for their detailed feedback and thoughtful suggestions during multiple stages of the project. The views expressed are those of the authors and do not necessarily reflect the official policies or positions of the supporting organizations.
{
    \small
    \bibliographystyle{ieeenat_fullname}
    \bibliography{main}
}
\appendix
\newpage
\maketitlesupplementary
\setcounter{page}{1}

\section{Detailed Related Work}
\label{app:related}

\subsection{Shortcut Learning in Vision-Language Models}
\looseness=-1 Deep learning models frequently learn decision rules that exploit spurious correlations in training data rather than the intended causal logic. This phenomenon is often referred to as the \textit{Clever Hans effect}, named after a horse that appeared to perform arithmetic but was later discovered to be responding to subtle cues from human observers rather than actually solving the task \cite{ye2025cleverhansmiragecomprehensive}. In machine learning, this effect describes models that appear to perform a task correctly while in reality relying on unintended signals or dataset artifacts.

In the context of Visual Question Answering (VQA), models often rely on language-only priors to predict answers without actually looking at the image \cite{rahmanzadehgervi2025visionlanguagemodelsblind}. Recent research identifies various shortcuts that artificially inflate performance on multimodal benchmarks: i) \textbf{Linguistic and Background Statistics:} Models frequently rely on background correlations rather than the primary object's features to make predictions \cite{Geirhos_2020, 11094344}; ii) \textbf{Keyword Shortcuts:} Models often learn a shallow, direct mapping between a noun, color, or shape and a high-probability answer, entirely bypassing the relational or spatial context \cite{si-etal-2022-language, li2025devildetailstacklingunimodal}. While these shortcuts yield high accuracy on identically distributed test sets, they cause catastrophic failure in out-of-distribution or adversarial scenarios. Our work directly addresses this by mathematically isolating these keyword and background shortcuts within the \method generation algorithm.

\subsection{Explainability Guided Training}
\looseness=-1 Existing research \cite{Shao_Skryagin_Stammer_Schramowski_Kersting_2021, ross2017rightrightreasonstraining} highlights that achieving high accuracy is insufficient if the underlying decision-making process relies on spurious correlations. Explainability guided training addresses this by incorporating auxiliary constraints during optimization to align model reasoning with human expectations. In vision-language domains, this often takes the form of visual grounding supervision, where models are penalized if their attention or predicted bounding boxes do not align with ground-truth causal features \cite{kamath2021mdetrmodulateddetection}. 

\looseness=-1 By explicitly supervising the intermediate representations, researchers can mitigate shortcut reliance. In the context of \method, we utilize this paradigm not just to improve accuracy, but as a controlled mechanism to train our diagnostic models. By modulating the grounding loss and dynamically weighting the cross-entropy loss, we successfully enforce the behavioral divergence required to train both $M_{\text{pure}}$ (which strictly aligns with causal spatial-relational geometry) and $M_{\text{spur}}$ (which bypasses spatial grounding to exploit injected statistical traps).

\subsection{Detailed Mechanics of Evaluated MxAI Methods}
Existing research \cite{dang2024explainableinterpretablemultimodallarge, sun2024reviewmultimodalexplainableartificial} highlights the growing need for explainability techniques. We broadly classify these methods based on their attribution scope, dividing them into local (micro-level) and global (macro-level) taxonomies:

\noindent\textbf{1. Local (Micro-Level) Attribution Methods.}
These methods focus on fine-grained attributions, seeking to explain model predictions by isolating the specific contributions of individual image patches, pixels, or text tokens, as well as the localized interactions between them.
\begin{itemize}[leftmargin=*]
    \item \textbf{\multishap \cite{wang2026multishapshapleybasedframeworkexplaining}:} A model-agnostic framework that leverages the Shapley Interaction Index to attribute predictions to pairwise interactions between fine-grained visual and textual elements (such as image patches and text tokens).
    \item \textbf{\dime \cite{lyu2022dimefinegrainedinterpretationsmultimodal}:} DIME enables fine-grained analysis by explicitly disentangling model behavior into unimodal contributions (UC) and multimodal interactions (MI). It is designed to maintain generality across arbitrary modalities, model architectures, and tasks, allowing stakeholders to understand model behavior and performing debugging.
    \item \textbf{\multiviz \cite{liang2023multivizvisualizingunderstandingmultimodal}:} A framework for visualizing and analyzing model behavior by scaffolding interpretability into four stages: (1) unimodal importance, (2) cross-modal interactions, (3) multimodal representations, and (4) multimodal prediction.  These complementary stages enable users to simulate predictions, assign interpretable concepts to features, perform error analysis, and debug models.
    \item \textbf{PixelSHAP + TokenSHAP \cite{goldshmidt2025attentionpleasepixelshapreveals, goldshmidt2024tokenshapinterpretinglargelanguage}:} A model-agnostic multimodal attribution framework that combines TokenSHAP for textual analysis and PixelSHAP for visual reasoning to produce joint explanations for vision–language model predictions. TokenSHAP attributes the model's output to individual tokens or substrings within the input prompt using Shapley values, modeling each token as a cooperative player whose contribution is estimated through Monte Carlo sampling over token subsets. PixelSHAP extends the same Shapley-value framework to visual inputs by treating segmented image regions or objects as players in the cooperative game and quantifying their influence on the model output through systematic perturbations of these visual components.
\end{itemize}

\noindent\textbf{2. Global (Macro-Level) Attribution Methods.}
Rather than focusing on individual tokens or patches, these methods evaluate the overall contribution of entire modalities and the high-level synergistic interactions between the visual and textual inputs as a whole.
\begin{itemize}[leftmargin=*]
    \item \textbf{InterSHAP \cite{Wenderoth_2025}:} A game-theoretic approach that shifts the focus from local tokens to the overall contribution of entire modalities, specifically isolating the synergistic interactions between the visual and textual inputs as a whole.
    \item \textbf{EMAP \cite{hessel2020doesmultimodalmodellearn}:} A diagnostic tool designed to identify whether cross-modal interactions genuinely improve model performance. It utilizes an empirical multimodality-additive function projection to modify model predictions such that cross-modal interactions are eliminated, isolating the additive, unimodal structure. This reveals whether a high-performing black-box model actually utilizes complex cross-modal reasoning or mostly exploits unimodal signals in the data.
    \item \textbf{Partial Information Decomposition (PID)\cite{williams2010nonnegativedecompositionmultivariateinformation}:} An information-theoretic framework for analyzing how multiple input sources jointly contribute to a model's prediction. PID decomposes the mutual information between a set of inputs and an output into distinct components representing unique, redundant, and synergistic information. In multimodal settings, this decomposition enables quantifying how much predictive information is provided exclusively by the visual modality, exclusively by the textual modality, redundantly by both, or only through their joint interaction. PID provides a principled way to measure whether a multimodal model truly relies on cross-modal reasoning rather than unimodal shortcuts.
\end{itemize}

\section{Dataset Generation Details and Proofs}
\label{app:proof}

\subsection{Case-0: Answer Prior Bias Elimination}
\label{app:case0_balancing}
As introduced in Section~\ref{sec:spurious}, the Answer Prior Bias (Case-0) heuristic assumes the model ignores the visual modality entirely and predicts $\hat{y}$ by exploiting the marginal distribution of answers $P(\mathcal{Y})$. In natural VQA, answers are heavily skewed, allowing models to achieve artificially high accuracy via prior guessing. In \method, we mitigate this through explicit dataset balancing. As verified by our dataset statistics, for \textit{Form 1} (Existence) queries, the generative engine enforces strict target-distractor independence, yielding a near-perfect balanced distribution across all buckets. 

For \textit{Form 0} (Counting) queries, placing multiple objects satisfying complex spatial constraints on a finite grid inherently introduces a structural skew. For instance, high depth queries in sparse environments (\textit{D3\_M\_F0\_$d_{0.3}$}) yield ``1'' as the most frequent answer ($37.2\%$). Thus, a pure Case-0 heuristic is strictly bounded by this empirical ceiling, making it insufficient to solve the dataset.

Our generation algorithm guarantees that the marginal answer distributions are almost identical across the two splits: $P(\mathcal{Y})_{Pure} = P(\mathcal{Y})_{Spurious}$. Since the Case-0 heuristic space is almost identical between the two environments, prior bias cannot account for any downstream divergence in behavior between the Pure and Spurious testbeds.

\begin{figure*}[!t]
    \centering
    \includegraphics[width=\textwidth]{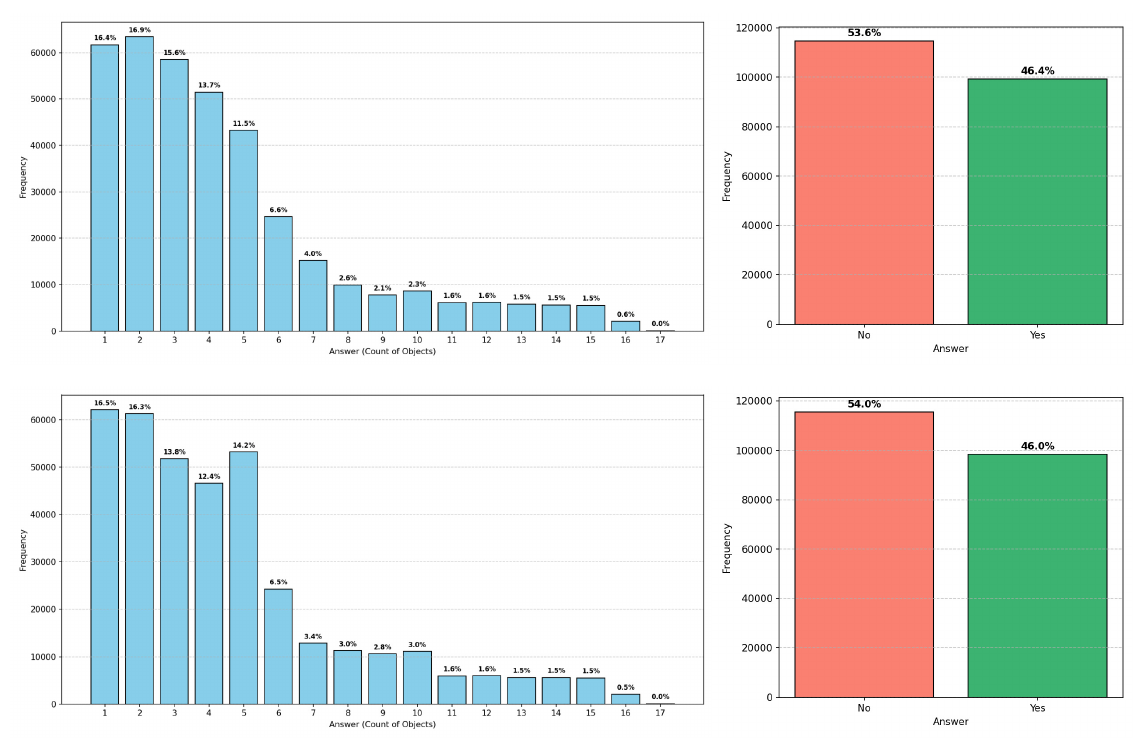}
    \caption{\textbf{Answer distribution in the training set.} Top row shows the answer distribution for the $\mathcal{D}_\text{pure}$ , while the bottom row shows the answer distribution for the $\mathcal{D}_\text{spur}$. Left plots correspond to the counting task (Form0) with answers ranging from 1–17, and right plots correspond to the binary yes no task (Form1).}
    \label{fig:dataset_answer_priors}
\end{figure*}

\subsection{Case-1: The Spatial Shortcut (Bag-of-Words) Elimination}
\label{app:case1_shortcut}
Our empirical validation confirms the mathematical design of the dataset splits regarding the spatial shortcut. Predicting the target count based solely on the bag-of-words spatial shortcut yields a 100\% success rate ($P = 1.00$) in $\mathcal{D}_{\text{spur}}$, successfully embedding the behavioral trap. Conversely, in $\mathcal{D}_{\text{pure}}$, the deterministic injection of adversarial distractors collapses this shortcut's success rate to 36.07\%, proving that a model mathematically cannot converge on this split without computing cross-modal spatial synergy.

\subsection{Case-2: Visual Feature Dominance Elimination}
\label{app:case2_balancing}

The Visual Feature Dominance heuristic assumes the model ignores the textual query $\mathcal{Q}$ entirely and predicts the answer based solely on layout imbalances inherent in the image $\mathcal{V}$. Two primary sub-hypotheses exist within this family: the Salience Prior ($H_{sal}$), where the presence of a specific attribute correlates with an answer, and the Majority Class Prior ($H_{\text{maj}}$), where the model blindly predicts based on the most frequent object type.

\looseness=-1 To prevent any single attribute from becoming a predictive shortcut ($H_{sal}$), our generation pipeline utilizes independent, uniform multi-randomization. During semantic instantiation, the exact values for color, shape, and direction are drawn from uniform distributions. Because the marginal probability of any specific visual feature occurring is statistically decorrelated from the ground-truth answer, $H_{sal}$ provides no reliable predictive signal. To nullify the visual majority dominance ($H_{\text{maj}}$), the generative engine must logically decorrelate the target from the visual majority, satisfying Property 2 (Orthogonality of Frequency). We structurally counter the majority heuristic through two distinct mechanisms. 

First, for all buckets in $\mathcal{D}_{\text{spur}}$ and the non-relational Attribute-Only buckets in $\mathcal{D}_{\text{pure}}$, the generator intervenes during noise injection. As outlined in Algorithm~\ref{alg:balancing}, it utilizes a stochastic branching mechanism parameterized over the atomic feature space.
\begin{algorithm}[h]
\caption{\textbf{Stochastic Target-Distractor Balancing}}
\label{alg:balancing}
\begin{algorithmic}[1]
    \Require Target attribute set $\mathcal{A}_{\text{target}}$, Total grid capacity $C$
    \State $u \sim \mathcal{U}(0, 1)$ \Comment{Sample from standard uniform distribution}
    \If{$u < 0.5$}
        \State \Comment{\textbf{BOOST Branch:} Force target to be the majority class}
        \State $N_{\text{target}} \gets \text{Sample}(\text{Upper Range})$
        \State $\text{Inject}(N_{\text{target}}, \mathcal{A}_{\text{target}})$
    \Else
        \State \Comment{\textbf{SUPPRESS Branch:} Force a disjoint distractor to be the majority}
        \State $N_{\text{target}} \gets \text{Sample}(\text{Lower Range})$
        \State $\mathcal{A}_{\text{dist}} \gets \text{Sample}(\mathcal{V}_{\text{attributes}} \setminus \mathcal{A}_{\text{target}})$
        \State $N_{\text{dist}} \gets \text{Sample}(> N_{\text{target}})$
        \State $\text{Inject}(N_{\text{target}}, \mathcal{A}_{\text{target}})$
        \State $\text{Inject}(N_{\text{dist}}, \mathcal{A}_{\text{dist}})$
    \EndIf
\end{algorithmic}
\end{algorithm}
While the theoretical split is uniform, downstream generative constraints (\eg valid grid space exhaustion and integer bounding) modulate this distribution. This active balancing bounds the success rate of the $H_{\text{maj}}$ heuristic to an empirical average of $\sim 44\%$ across evaluated buckets, rendering it statistically insufficient.

\looseness=-1 Second, for the relational spatial queries in $\mathcal{D}_{\text{pure}}$, the $H_{\text{maj}}$ heuristic is destroyed as an emergent consequence of eliminating the spatial shortcut (Case-1). By injecting adversarial distractors that share the target's visual attributes but explicitly lie outside the valid bounding region, the target's visual class almost always becomes the majority class (\eg $99.36\%$ frequency in \textit{D1\_CO\_F0\_$d_{0.3}$}). However, because these distractors violate the spatial constraint, the global count of the majority class inherently exceeds the true relational ground-truth count. A model attempting to use $H_{\text{maj}}$ will consistently overcount and fail, safely eliminating the heuristic.

\subsection{Case-3: Logical Decomposition (Partial Logic) Elimination}
\label{app:case3_proof}

\looseness=-1 The Logical Decomposition heuristic occurs when a model receives a multi-hop compositional query (e.g., Depth 2 or Depth 3) but evaluates only a subset of the spatial constraints. Let a complex query require the logical intersection of $M$ spatial constraints relative to $M$ anchors: $\mathcal{Q} = \bigwedge_{i=1}^M C_i$. A model exploiting partial logic will drop constraint $k$, effectively treating the intersection as $\bigwedge_{i \neq k} C_i$. In $\mathcal{D}_{\text{spur}}$, the model never encounters this heuristic because it collapses entirely to the Case-1 Bag-of-Words shortcut, bypassing spatial evaluation altogether. However, in $\mathcal{D}_{\text{pure}}$, where Case-1 is mathematically eliminated, the model is forced to evaluate spatial regions. If the dataset allows the set of objects satisfying $\bigwedge_{i \neq k} C_i$ to frequently be a subset of those satisfying $C_k$, the model will learn to safely ignore anchor $k$ without penalty. To definitively eliminate this heuristic in the Pure dataset, we generalize the Robust Intersection requirement to $M$ dimensions (Theorem 1). The formal proof expanding on the main text's proof sketch is as follows.

\noindent\textbf{Proof of Theorem 1.} Let $Y_{\text{true}}$ be the set of valid target objects that correctly answer the multimodal query. By definition, these objects must possess the target visual attributes and reside entirely within the strictly valid spatial intersection of all anchors: 
$$V_{\text{target}} = \bigcap_{i=1}^M V_i$$
A model employing the partial logic heuristic drops spatial constraint $k$ to simplify computation, thus evaluating the relaxed geometric region:
$$V_{-k} = \bigcap_{i \neq k} V_i$$
Let $Y_{\text{partial}}$ be the candidate set of objects detected by this heuristic. Because the intersection of a subset of regions is always a superset of the intersection of all regions, it geometrically holds that $V_{\text{target}} \subseteq V_{-k}$. Consequently, $Y_{\text{true}} \subseteq Y_{\text{partial}}$. The set of false positives detected by the heuristic is strictly defined by the relative complement of these spatial regions:
$$Y_{\text{FP}} \subset \left( V_{-k} \setminus V_{\text{target}} \right)$$
We can simplify the false positive bounding region as:
$$V_{-k} \setminus V_{\text{target}} = V_{-k} \setminus \left( V_{-k} \cap V_k \right) = V_{-k} \setminus V_k = C_k$$
Thus, any false positive objects must reside within the confuser region $C_k$. By algorithmically guaranteeing that $C_k$ is non-empty and contains at least one adversarial object $d_{\text{adv}}$ matching the target's visual attributes, the heuristic will blindly detect it, guaranteeing that $|Y_{\text{FP}}| \ge 1$. Therefore, the total count of objects detected by the heuristic is:
$$|Y_{\text{partial}}| = |Y_{\text{true}}| + |Y_{\text{FP}}| > |Y_{\text{true}}|$$
Because the heuristic strictly overcounts the true targets, it yields an incorrect prediction, incurs a strict training loss penalty, and is mathematically eliminated. \hfill $\blacksquare$

\noindent\textbf{Adversarial Confuser Placement Implementation.} 
Rather than relying on random rejection sampling to fulfill this theorem, the $\mathcal{D}_{\text{pure}}$ generation engine actively constructs these adversarial confuser regions during the target placement phase. As detailed in Algorithm~\ref{alg:confuser}, the engine computes $C_k$ for every anchor in the query.
\begin{algorithm}[H]
\caption{\textbf{Adversarial Confuser Placement for Depth $\ge 2$}}
\label{alg:confuser}
\begin{algorithmic}[1]
    \Require List of Anchor Valid Regions $[V_1, V_2, \dots, V_M]$, Target Attributes $\mathcal{A}_{\text{target}}$
    \State \Comment{\textbf{Phase 1: Guaranteed Confuser Placement}}
    \For{$k = 1 \dots M$}
        \State $V_{-k} \gets \bigcap_{i \neq k} V_i$ \Comment{Intersection of all OTHER anchors}
        \State $C_k \gets V_{-k} \setminus V_k$ \Comment{Subtract the dropped anchor's region}
        \If{$C_k \neq \emptyset$}
            \State $pos \gets \text{Sample}(C_k)$
            \State $\text{Inject}(pos, \mathcal{A}_{\text{target}})$ \Comment{Place target-attribute object in confuser region}
        \EndIf
    \EndFor
    \State \Comment{\textbf{Phase 2: Stochastic Filler Overload}}
    \While{$\text{Noise Target Not Met}$}
        \If{$\text{random}() < 0.75$ \textbf{and} $\text{Sample From Outside Valid Region}$}
            \State $k \gets \text{SampleRandomAnchor}()$
            \State $\text{Inject}(\text{Sample}(C_k), \mathcal{A}_{\text{target}})$
        \EndIf
    \EndWhile
\end{algorithmic}
\end{algorithm}
Before injecting general background noise, the engine guarantees the placement of exactly one target-attribute object into every valid confuser region $C_k$. During the subsequent filler loop, the engine overwhelmingly ($75\%$ probability) targets these specific $C_k$ regions when injecting adversarial spatial distractors. 

Our empirical reports rigorously validate this active confuser placement. When a heuristic drops a spatial constraint, its predictive success rate plummets. For instance, in dense Depth 2 environments (\textit{D2\_M\_F0\_$d_{0.7}$}), dropping either anchor 1 or anchor 2 yields a predictive ratio of only $16.40\%$ and $16.87\%$, respectively. Consequently, minimizing the training loss requires the strict evaluation of the full compositional intersection $\bigwedge_{i=1}^M C_i$, effectively neutralizing the Logical Decomposition heuristic.

\subsection{Collision-Free Object Placement Requirements}
To ensure that spatial relationships are unambiguous and $C_k$ regions are geometrically sound, the generator employs a strictly collision-free placement algorithm. When instantiating the set of anchor objects $\mathcal{A}$ and target objects $\mathcal{T}_{\text{tgt}}$, the grid is modeled as a discrete 2D matrix. Anchors are placed sequentially; for multi-hop queries ($D > 1$), subsequent anchors are placed such that their valid intersecting region $V_{\text{target}} = \bigcap V_i$ contains at least one available grid cell. To prevent edge-case spatial ambiguities, we enforce a strict minimum grid-cell distance margin between objects involved in directional relationships.

\section{Dataset Details}
\label{app:dataset}
Figures~\ref{fig:d03_pure_examples}–\ref{fig:d07_spur_examples} illustrate representative examples from the $\mathcal{D}_{\text{pure}}$ and $\mathcal{D}_{\text{spur}}$ datasets across different reasoning depths, forms, question types, and scene densities. Each panel shows an image–question pair along with the corresponding answer. 

\section{Additional Results}

\subsection{Model Accuracies}
Refer to Tables~\ref{tab:model_acc_d03} and~\ref{tab:model_acc_d07} for the accuracies of existing open-source models compared to our trained models, $\mathcal{M}_{\text{pure}}$ and $\mathcal{M}_{\text{spur}}$.

\subsection{Local Explainability Algorithms}
Tables~\ref{tab:rma_pure},~\ref{tab:iou_pure},~\ref{tab:rma_spur}, and~\ref{tab:iou_spur} give the bucket-wise performance of the three algorithms on the dataset.

\subsection{Global Explainability Algorithms}
Figures~\ref{fig:pureds_puremodel_examples}--\ref{fig:intershap_spur} illustrate attribution mass distributions and interaction-based explanations for both $\mathcal{M}_{\text{pure}}$ and $\mathcal{M}_{\text{spur}}$ across datasets.

\section{Limitations and Future Work}
While GridVQA-X mathematically guarantees unique ground-truth explanations, it operates within a highly controlled 2D abstraction that fundamentally lacks the noisy feature distributions and deeply entangled semantic representations of real-world multimodal tasks. Additionally, the framework's scope is currently restricted to spatial-relational composition and basic attribute binding. It does not measure an explainer’s faithfulness when dealing with other complex cross-modal synergies, such as temporal reasoning, physical commonsense, or complex mathematical logic.

To bridge this abstraction gap, future work should extend these diagnostic principles to richer spatial reasoning benchmarks to provide a more rigorous stress test involving continuous spatial relations and realistic object interactions. Furthermore, since our empirical results expose that existing explainers often devolve into shallow detectors or hallucinate complex interactions, future research must shift toward developing novel explanation algorithms specifically designed to pass this zero-ambiguity testbed. Finally, extending the current evaluation pipeline to explicitly assess the faithfulness of natural language rationales and chain-of-thought explanations generated directly by multimodal LLMs presents a highly promising direction.

\newpage
\begin{figure*}[t]
    \centering
    \includegraphics[width=0.83\linewidth]{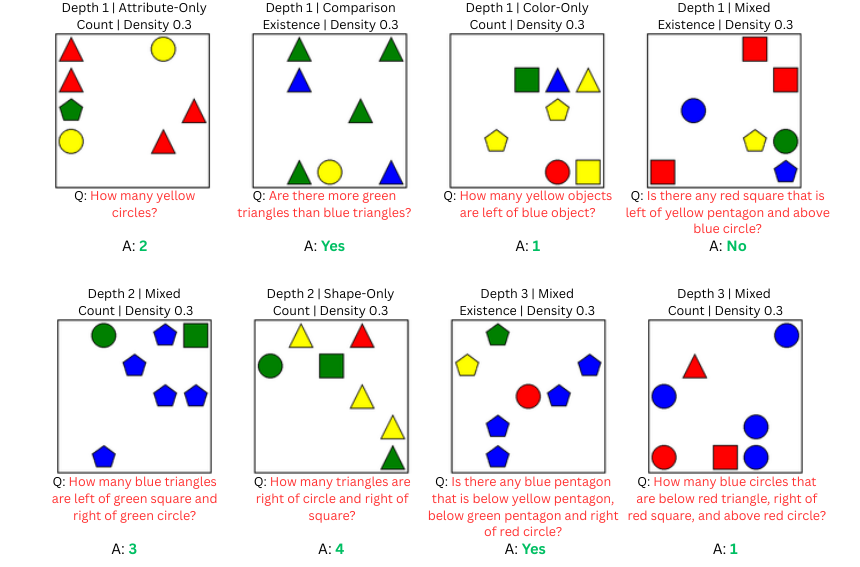}
    \caption{Example samples from $\mathcal{D}_{\text{pure}}$ with density $d=0.3$.}
    \label{fig:d03_pure_examples}
\end{figure*}

\begin{figure*}
    \centering
    \includegraphics[width=0.83\linewidth]{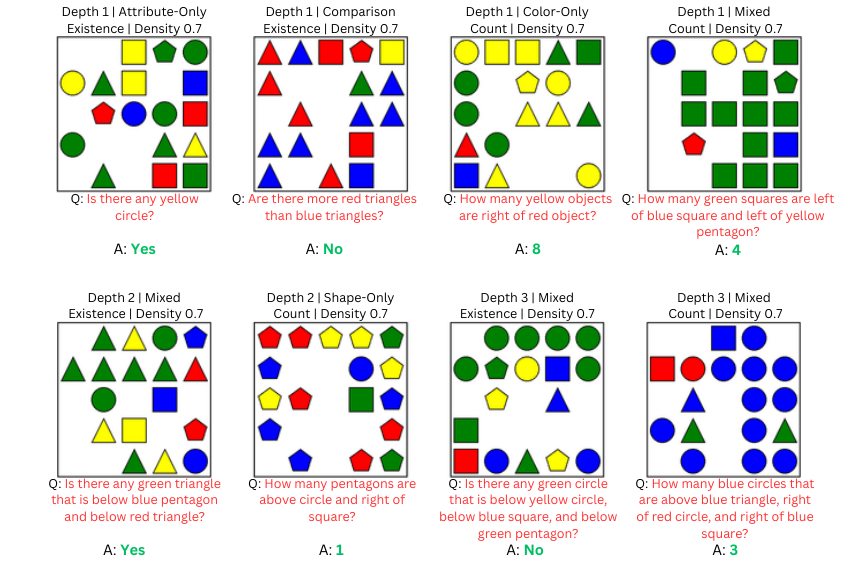}
    \caption{Example samples from $\mathcal{D}_{\text{pure}}$ with density $d=0.7$.}
    \label{fig:d07_pure_examples}
\end{figure*}

\clearpage
\begin{figure*}
    \centering
    \includegraphics[width=0.83\linewidth]{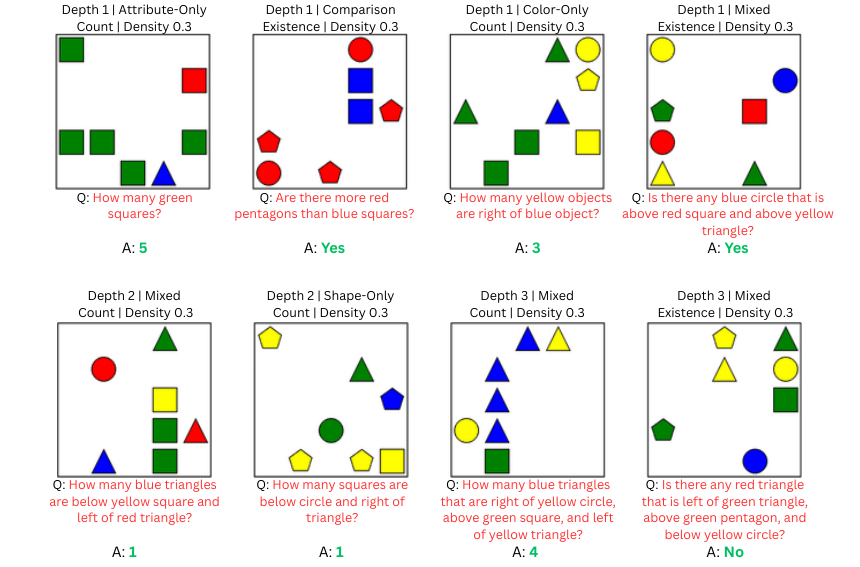}
    \caption{Example samples from $\mathcal{D}_{\text{spur}}$ with density $d=0.3$.
    Note that the correct answer can be inferred even when the directional cues are ignored.}
    \label{fig:d03_spur_examples}
\end{figure*}

\begin{figure*}
    \centering
    \includegraphics[width=0.83\linewidth]{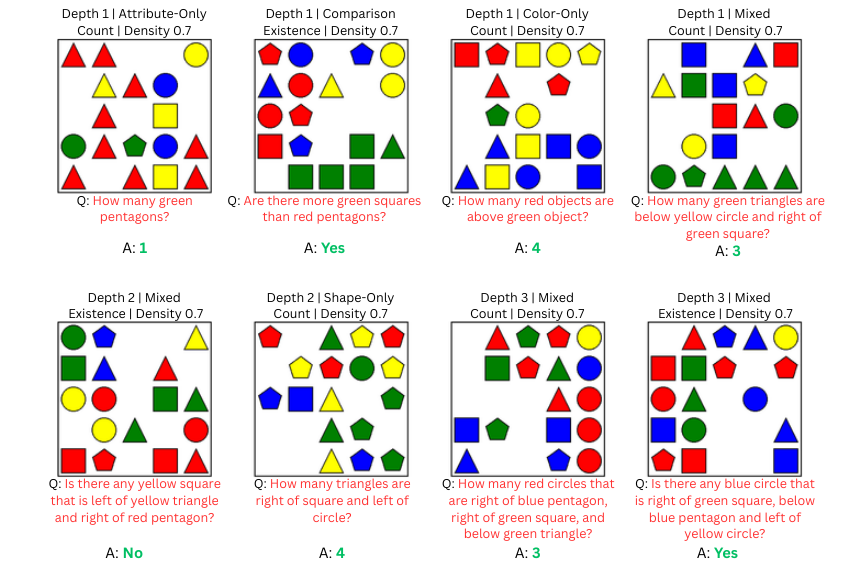}
    \caption{Example samples from $\mathcal{D}_{\text{spur}}$ with density $d=0.7$.
    Note that the correct answer can be inferred even when the directional cues are ignored.}
    \label{fig:d07_spur_examples}
\end{figure*}
\clearpage

\begin{table*}[htbp]
\centering

\setlength{\tabcolsep}{9pt}      
\renewcommand{\arraystretch}{0.9} 
\caption{Model Accuracies for Open Source Models on $\mathcal{D}_{\text{pure}}$ with $d_{0.3}$}
\label{tab:model_acc_d03}
\begin{tabular}{llcccccc}
\toprule
Bucket & Depth & Qwen2.5-VL-3B & Gemma-3-4B & Llava-1.5-7B & Qwen3-VL-30B & $\mathcal{M}_{\text{spur}}$ & $\mathcal{M}_{\text{pure}}$ \\
\midrule
\multicolumn{8}{l}{\textbf{Form 0}}\\
A  & D1 & 68.0 & 48.0 & 16.0 & 70.0 & 100 & \textbf{100} \\
SO & D1 & 20.0 & 18.0 & 18.0 & 24.0 & 32 & \textbf{100} \\
SO & D2 & 22.0 & 34.0 & 22.0 & 32.0 & 16 & \textbf{100} \\
CO & D1 & 34.0 & 20.0 & 14.0 & 28.0 & 18 & \textbf{100} \\
CO & D2 & 22.0 & 38.0 & 22.0 & 38.0 & 16 & \textbf{100} \\
M  & D1 & 32.0 & 24.0 & 22.0 & 20.0 & 36 & \textbf{100} \\
M  & D2 & 28.0 & 30.0 & 28.0 & 28.0 & 8 & \textbf{100} \\
M  & D3 & 32.0 & 38.0 & 40.0 & 34.0 & 14 & \textbf{100} \\
CMP & D1 & 76.0 & 50.0 & 52.0 & 72.0 & 100 & \textbf{100} \\
CMP & D2 & 66.0 & 44.0 & 54.0 & 66.0 & 96 & \textbf{100} \\
CMP & D3 & 46.0 & 42.0 & 38.0 & 50.0 & 96 & \textbf{100} \\
\midrule
\multicolumn{8}{l}{\textbf{Form 1}}\\
A  & D1 & 94.0 & 90.0 & 50.0 & 94.0 & 100 & \textbf{100} \\
SO & D1 & 48.0 & 44.0 & 60.0 & 80.0 & 46 & \textbf{100} \\
SO & D2 & 58.0 & 64.0 & 56.0 & 60.0 & 60 & \textbf{100} \\
CO & D1 & 56.0 & 52.0 & 58.0 & 78.0 & 50 & \textbf{100} \\
CO & D2 & 62.0 & 60.0 & 58.0 & 74.0 & 56 & \textbf{100} \\
M  & D1 & 64.0 & 62.0 & 54.0 & 70.0 & 62 & \textbf{100} \\
M  & D2 & 48.0 & 48.0 & 46.0 & 68.0 & 46 & \textbf{100} \\
M  & D3 & 56.0 & 56.0 & 58.0 & 40.0 & 56 & \textbf{100} \\
\bottomrule
\end{tabular}
\end{table*}

\begin{table*}[t]
\centering

\setlength{\tabcolsep}{9pt}      
\renewcommand{\arraystretch}{0.9} 
\caption{Model Accuracies for Open Source Models on $\mathcal{D}_{\text{pure}}$ with $d_{0.7}$}
\label{tab:model_acc_d07}
\begin{tabular}{llcccccc}
\toprule
Bucket & Depth & Qwen2.5-VL-3B & Gemma-3-4B & Llava-1.5-7B & Qwen3-VL-30B & $\mathcal{M}_{\text{spur}}$ & $\mathcal{M}_{\text{pure}}$ \\
\midrule
\multicolumn{8}{l}{\textbf{Form 0}}\\
A  & D1 & 30.0 & 14.0 & 14.0 & 36.0 & 100 & \textbf{100} \\
SO & D1 & 22.0 & 10.0 & 10.0 & 12.0 & 2 & \textbf{100} \\
SO & D2 & 18.0 & 22.0 & 18.0 & 18.0 & 0 & \textbf{100} \\
CO & D1 & 22.0 & 4.0 & 0.0 & 6.0 & 6 & \textbf{100} \\
CO & D2 & 12.0 & 20.0 & 22.0 & 10.0 & 2 & \textbf{100} \\
M  & D1 & 16.0 & 8.0 & 6.0 & 14.0 & 6 & \textbf{100} \\
M  & D2 & 12.0 & 22.0 & 28.0 & 10.0 & 2 & \textbf{100} \\
M  & D3 & 32.0 & 24.0 & 30.0 & 10.0 & 0 & \textbf{100} \\
CMP & D1 & 66.0 & 42.0 & 42.0 & 76.0 & \textbf{100} & 98 \\
CMP & D2 & 52.0 & 14.0 & 48.0 & 36.0 & 92 & \textbf{100} \\
CMP & D3 & 28.0 & 18.0 & 24.0 & 20.0 & 80 & \textbf{100} \\
\midrule
\multicolumn{8}{l}{\textbf{Form 1}}\\
A  & D1 & 92.0 & 70.0 & 48.0 & 88.0 & 100 & \textbf{100} \\
SO & D1 & 46.0 & 46.0 & 60.0 & 68.0 & 46 & \textbf{100} \\
SO & D2 & 62.0 & 60.0 & 56.0 & 48.0 & 60 & \textbf{100} \\
CO & D1 & 50.0 & 50.0 & 54.0 & 70.0 & 50 & \textbf{100} \\
CO & D2 & 56.0 & 56.0 & 56.0 & 56.0 & 56 & \textbf{100} \\
M  & D1 & 58.0 & 56.0 & 60.0 & 78.0 & 56 & \textbf{100} \\
M  & D2 & 50.0 & 52.0 & 48.0 & 66.0 & 46 & \textbf{100} \\
M  & D3 & 58.0 & 44.0 & 54.0 & 50.0 & 56 & \textbf{100} \\
\bottomrule
\end{tabular}
\end{table*}

\begin{table*}[t]
\centering

\setlength{\tabcolsep}{9pt}      
\renewcommand{\arraystretch}{0.9} 
\caption{\textbf{RMA ($\uparrow$) comparison on $M_{\text{pure}}$ with $\mathcal{D}_{\text{pure}}$.} Note the severe performance drop across all methods on Existence tasks (Form 1) compared to Counting tasks (Form 0), highlighting a structural reliance on mask volume rather than precise logical routing.}
\label{tab:rma_pure}
\begin{tabular}{llcccccc}
\toprule
 &  & \multicolumn{2}{c}{DIME} & \multicolumn{2}{c}{MultiSHAP} & \multicolumn{2}{c}{MultiViz} \\
Bucket & Depth & $d_{0.3}$ & $d_{0.7}$ & $d_{0.3}$ & $d_{0.7}$ & $d_{0.3}$ & $d_{0.7}$ \\
\midrule
\multicolumn{8}{l}{\textbf{Form 0}}\\
A  & D1 & 26.7\err{9.8} & 41.4\err{13.6} & 66.7\err{29.9} & 69.2\err{21.6} & 47.3\err{15.1} & 56.4\err{17.0}\\
SO & D1 & 25.8\err{10.1} & 36.1\err{15.6} & 73.5\err{20.5} & 61.4\err{23.2} & 51.3\err{11.7} & 48.6\err{14.4}\\
SO & D2 & 26.3\err{8.2} & 30.5\err{12.9} & 64.7\err{9.9} & 49.5\err{17.9} & 49.7\err{11.8} & 45.4\err{13.6}\\
CO & D1 & 26.9\err{9.6} & 40.4\err{14.3} & 70.2\err{14.4} & 66.7\err{12.7} & 47.9\err{13.3} & 50.7\err{14.1}\\
CO & D2 & 27.6\err{7.3} & 32.1\err{12.6} & 68.1\err{15.7} & 54.2\err{15.7} & 48.0\err{11.8} & 51.8\err{12.4}\\
M  & D1 & 29.6\err{9.7} & 35.9\err{14.4} & 60.4\err{24.0} & 61.4\err{17.7} & 47.9\err{14.6} & 45.6\err{11.4}\\
M  & D2 & 31.2\err{9.5} & 30.0\err{11.4} & 70.4\err{10.8} & 53.8\err{14.5} & 47.6\err{12.7} & 46.7\err{14.8}\\
M  & D3 & 31.1\err{9.4} & 27.9\err{7.0} & 68.3\err{10.9} & 47.1\err{15.4} & 50.8\err{13.5} & 44.9\err{9.8}\\
CMP & D1 & 35.8\err{4.9} & 56.0\err{12.5} & 87.4\err{5.8} & 79.2\err{8.4} & 46.0\err{13.6} & 55.7\err{13.2}\\
CMP & D2 & 34.2\err{5.9} & 44.0\err{10.4} & 83.9\err{4.0} & 71.6\err{12.5} & 49.6\err{11.7} & 49.6\err{15.7}\\
CMP & D3 & 35.4\err{4.8} & 38.2\err{7.6} & 82.6\err{3.3} & 62.0\err{9.0} & 52.2\err{10.2} & 52.4\err{11.4}\\
\midrule
\multicolumn{8}{l}{\textbf{Form 1}}\\
A  & D1 & 23.1\err{9.7} & 30.4\err{15.2} & 56.4\err{15.9} & 52.9\err{22.2} & 42.9\err{18.5} & 48.3\err{25.1}\\
SO & D1 & 12.6\err{8.6} & 16.5\err{17.0} & 45.1\err{20.9} & 32.4\err{16.3} & 28.6\err{16.5} & 31.3\err{19.4}\\
SO & D2 & 17.1\err{6.7} & 18.2\err{9.1} & 50.3\err{15.5} & 35.9\err{11.9} & 37.8\err{13.6} & 30.5\err{13.4}\\
CO & D1 & 16.3\err{10.3} & 19.0\err{13.5} & 44.9\err{19.2} & 28.6\err{10.3} & 32.1\err{13.5} & 30.6\err{18.4}\\
CO & D2 & 20.8\err{9.9} & 20.1\err{9.7} & 51.4\err{16.9} & 29.7\err{9.1} & 40.4\err{14.6} & 36.8\err{11.1}\\
M  & D1 & 17.4\err{9.4} & 19.1\err{13.2} & 43.3\err{20.5} & 31.9\err{16.5} & 34.8\err{18.4} & 31.7\err{16.4}\\
M  & D2 & 20.8\err{7.4} & 18.7\err{6.8} & 37.3\err{13.4} & 34.2\err{12.8} & 36.1\err{13.8} & 39.1\err{19.4}\\
M  & D3 & 24.3\err{6.8} & 24.3\err{8.1} & 43.7\err{16.2} & 36.4\err{9.4} & 44.9\err{11.7} & 40.8\err{13.5}\\
\bottomrule
\end{tabular}
\end{table*}

\begin{table*}[t]
\centering
\setlength{\tabcolsep}{9pt}      
\renewcommand{\arraystretch}{0.9} 
\caption{\textbf{IoU ($\uparrow$) comparison on $M_{\text{pure}}$ with $\mathcal{D}_{\text{pure}}$.} MultiSHAP demonstrates strong object grounding, while DIME exhibits poor precision due to systemic attribution spread, and MultiViz suffers a complete spatial grounding collapse ($<10\%$ across all depths).}
\label{tab:iou_pure}
\begin{tabular}{llcccccc}
\toprule
 &  & \multicolumn{2}{c}{DIME} & \multicolumn{2}{c}{MultiSHAP} & \multicolumn{2}{c}{MultiViz} \\
Bucket & Depth & $d_{0.3}$ & $d_{0.7}$ & $d_{0.3}$ & $d_{0.7}$ & $d_{0.3}$ & $d_{0.7}$ \\
\midrule
\multicolumn{8}{l}{\textbf{Form 0}}\\
A  & D1 & 28.6\err{14.4} & 27.4\err{10.5} & 71.3\err{31.6} & 62.1\err{22.8} & 7.7\err{3.7} & 5.3\err{3.4}\\
SO & D1 & 20.7\err{10.1} & 25.2\err{10.2} & 38.7\err{25.8} & 44.3\err{25.7} & 4.6\err{3.2} & 3.9\err{3.1}\\
SO & D2 & 21.4\err{8.0} & 22.2\err{7.9} & 55.7\err{10.7} & 35.9\err{21.9} & 4.0\err{2.7} & 4.1\err{2.9}\\
CO & D1 & 21.0\err{9.2} & 23.7\err{9.8} & 68.7\err{12.3} & 58.4\err{20.2} & 3.1\err{2.4} & 2.2\err{2.0}\\
CO & D2 & 19.6\err{6.4} & 21.8\err{8.2} & 53.7\err{18.7} & 56.8\err{11.5} & 3.6\err{2.3} & 2.9\err{2.5}\\
M  & D1 & 25.9\err{10.1} & 23.1\err{9.6} & 51.1\err{26.1} & 64.0\err{13.1} & 5.2\err{2.8} & 3.5\err{3.1}\\
M  & D2 & 22.7\err{8.3} & 20.5\err{7.3} & 59.1\err{21.4} & 60.1\err{15.3} & 4.4\err{2.9} & 3.6\err{2.2}\\
M  & D3 & 20.1\err{6.7} & 18.7\err{7.8} & 47.7\err{17.6} & 45.6\err{19.6} & 4.6\err{2.6} & 4.3\err{2.3}\\
CMP & D1 & 28.8\err{7.0} & 27.4\err{8.5} & 84.0\err{19.5} & 47.0\err{20.3} & 4.8\err{2.7} & 3.5\err{1.8}\\
CMP & D2 & 24.4\err{5.8} & 25.5\err{7.3} & 66.2\err{21.5} & 51.7\err{20.2} & 4.0\err{2.3} & 2.4\err{1.6}\\
CMP & D3 & 20.3\err{5.3} & 20.8\err{7.4} & 48.6\err{14.0} & 49.5\err{11.8} & 4.3\err{2.1} & 2.9\err{2.1}\\
\midrule
\multicolumn{8}{l}{\textbf{Form 1}}\\
A  & D1 & 31.3\err{14.4} & 28.0\err{13.1} & 73.0\err{9.4} & 58.7\err{30.2} & 9.0\err{3.4} & 7.3\err{4.8}\\
SO & D1 & 11.7\err{8.2} & 12.7\err{9.4} & 42.3\err{16.3} & 41.3\err{20.6} & 5.0\err{3.8} & 6.3\err{4.1}\\
SO & D2 & 14.5\err{6.9} & 15.6\err{6.6} & 37.0\err{12.0} & 31.2\err{15.6} & 4.8\err{2.8} & 5.6\err{3.1}\\
CO & D1 & 15.7\err{10.0} & 18.8\err{14.7} & 47.4\err{16.9} & 41.5\err{27.6} & 5.6\err{4.3} & 6.0\err{4.0}\\
CO & D2 & 14.4\err{7.1} & 14.6\err{7.6} & 47.4\err{19.4} & 39.6\err{15.3} & 4.4\err{3.3} & 5.0\err{3.2}\\
M  & D1 & 17.7\err{11.5} & 20.1\err{14.1} & 44.0\err{21.5} & 33.4\err{20.6} & 6.5\err{4.7} & 6.9\err{4.6}\\
M  & D2 & 15.5\err{6.6} & 15.6\err{7.8} & 33.6\err{12.2} & 43.0\err{19.9} & 5.6\err{2.4} & 5.2\err{3.0}\\
M  & D3 & 17.3\err{6.3} & 17.4\err{6.7} & 47.4\err{18.6} & 39.0\err{10.2} & 4.2\err{2.5} & 4.3\err{2.7}\\
\bottomrule
\end{tabular}
\end{table*}

\begin{table*}[t]
\centering

\setlength{\tabcolsep}{9pt}      
\renewcommand{\arraystretch}{0.9} 
\caption{\textbf{RMA ($\uparrow$) comparison on $M_{\text{spur}}$ with $\mathcal{D}_{\text{spur}}$.} MultiSHAP scores noticeably higher on this shortcut-reliant model than on the faithful $M_{\text{pure}}$, showing that game-theoretic methods may inherently favor shallow 1-to-1 feature matching over entangled spatial synergy.}
\label{tab:rma_spur}
\begin{tabular}{llcccccc}
\toprule
 &  & \multicolumn{2}{c}{DIME} & \multicolumn{2}{c}{MultiSHAP} & \multicolumn{2}{c}{MultiViz} \\
Bucket & Depth & $d_{0.3}$ & $d_{0.7}$ & $d_{0.3}$ & $d_{0.7}$ & $d_{0.3}$ & $d_{0.7}$ \\
\midrule
\multicolumn{8}{l}{\textbf{Form 0}}\\
A  & D1 & 23.5\err{7.0} & 41.2\err{15.0} & 75.2\err{20.7} & 85.2\err{4.7} & 36.9\err{20.5} & 44.9\err{18.1}\\
SO & D1 & 24.0\err{9.0} & 37.9\err{14.5} & 83.3\err{5.1} & 78.7\err{17.4} & 37.6\err{17.0} & 36.1\err{16.4}\\
SO & D2 & 26.0\err{8.0} & 34.7\err{14.3} & 77.7\err{10.8} & 69.3\err{17.0} & 49.4\err{18.5} & 53.6\err{18.6}\\
CO & D1 & 26.6\err{7.5} & 36.2\err{13.9} & 77.7\err{12.8} & 80.7\err{7.8} & 47.5\err{17.2} & 45.1\err{12.0}\\
CO & D2 & 25.4\err{7.8} & 33.1\err{10.4} & 78.3\err{7.9} & 81.0\err{10.2} & 53.0\err{13.8} & 51.8\err{12.6}\\
M  & D1 & 26.3\err{7.5} & 37.5\err{11.9} & 71.4\err{18.3} & 64.8\err{17.5} & 47.8\err{15.4} & 49.2\err{15.9}\\
M  & D2 & 27.6\err{8.0} & 33.0\err{9.8} & 71.4\err{13.0} & 64.8\err{16.5} & 57.5\err{15.6} & 50.8\err{13.3}\\
M  & D3 & 32.5\err{6.1} & 34.5\err{7.9} & 72.9\err{13.3} & 57.2\err{19.2} & 58.5\err{15.4} & 43.7\err{12.5}\\
CMP & D1 & 32.4\err{8.0} & 51.6\err{13.9} & 82.0\err{7.8} & 71.0\err{12.5} & 47.3\err{17.6} & 49.1\err{19.4}\\
CMP & D2 & 31.6\err{5.6} & 44.4\err{13.0} & 79.8\err{6.7} & 63.1\err{15.0} & 53.4\err{14.5} & 47.4\err{14.9}\\
CMP & D3 & 33.6\err{7.6} & 38.2\err{9.7} & 79.9\err{4.1} & 62.5\err{11.3} & 54.7\err{12.5} & 46.3\err{14.4}\\
\midrule
\multicolumn{8}{l}{\textbf{Form 1}}\\
A  & D1 & 19.4\err{9.1} & 32.4\err{17.5} & 52.1\err{16.8} & 59.5\err{26.2} & 24.7\err{15.5} & 28.6\err{17.4}\\
SO & D1 & 10.3\err{8.6} & 17.1\err{21.1} & 56.9\err{22.2} & 48.5\err{21.9} & 33.6\err{21.5} & 30.7\err{21.6}\\
SO & D2 & 14.2\err{7.8} & 14.7\err{7.8} & 71.3\err{14.3} & 61.8\err{19.9} & 51.8\err{21.9} & 45.4\err{16.7}\\
CO & D1 & 14.7\err{9.2} & 17.5\err{13.9} & 56.4\err{22.0} & 43.7\err{22.5} & 23.2\err{22.7} & 21.2\err{19.6}\\
CO & D2 & 18.6\err{9.6} & 21.1\err{12.8} & 62.0\err{12.7} & 56.5\err{11.5} & 44.2\err{20.5} & 34.1\err{16.7}\\
M  & D1 & 16.6\err{10.2} & 19.8\err{15.8} & 44.5\err{25.4} & 26.5\err{17.7} & 26.3\err{17.4} & 22.7\err{15.7}\\
M  & D2 & 17.3\err{8.4} & 16.8\err{10.8} & 46.3\err{12.8} & 29.4\err{18.0} & 39.9\err{17.1} & 25.2\err{14.7}\\
M  & D3 & 23.8\err{8.0} & 25.9\err{13.7} & 59.3\err{21.1} & 36.9\err{14.1} & 43.7\err{18.1} & 30.7\err{13.6}\\
\bottomrule
\end{tabular}
\end{table*}

\begin{table*}[t]
\centering

\setlength{\tabcolsep}{9pt}      
\renewcommand{\arraystretch}{0.9} 
\caption{\textbf{IoU ($\uparrow$) comparison on $M_{\text{spur}}$ with $\mathcal{D}_{\text{spur}}$.} MultiSHAP isolates the independent visual features driving the Bag-of-Words shortcut, whereas DIME and MultiViz consistently fail to generate coherent, object-centric bounding boxes regardless of the reasoning pathway.}
\label{tab:iou_spur}
\begin{tabular}{llcccccc}
\toprule
 &  & \multicolumn{2}{c}{DIME} & \multicolumn{2}{c}{MultiSHAP} & \multicolumn{2}{c}{MultiViz} \\
Bucket & Depth & $d_{0.3}$ & $d_{0.7}$ & $d_{0.3}$ & $d_{0.7}$ & $d_{0.3}$ & $d_{0.7}$ \\
\midrule
\multicolumn{8}{l}{\textbf{Form 0}}\\
A  & D1 & 25.1\err{12.6} & 25.1\err{9.1} & 66.1\err{20.6} & 56.7\err{24.5} & 9.3\err{5.1} & 5.1\err{3.9}\\
SO & D1 & 20.2\err{8.8} & 22.9\err{9.1} & 44.3\err{23.1} & 58.8\err{19.7} & 7.6\err{4.4} & 4.3\err{3.8}\\
SO & D2 & 20.5\err{8.1} & 23.3\err{10.1} & 48.0\err{20.4} & 33.4\err{17.2} & 8.2\err{4.3} & 6.7\err{3.9}\\
CO & D1 & 21.6\err{8.1} & 24.0\err{7.8} & 61.5\err{16.3} & 64.8\err{14.2} & 6.3\err{4.0} & 6.4\err{4.1}\\
CO & D2 & 16.6\err{5.8} & 19.7\err{6.7} & 46.7\err{21.6} & 63.7\err{11.9} & 7.1\err{3.3} & 6.8\err{3.3}\\
M  & D1 & 21.4\err{8.8} & 23.5\err{8.4} & 50.9\err{18.4} & 47.9\err{17.8} & 7.4\err{4.3} & 6.0\err{4.0}\\
M  & D2 & 18.6\err{7.3} & 20.9\err{8.1} & 40.9\err{21.1} & 36.0\err{21.9} & 7.9\err{3.8} & 5.8\err{3.2}\\
M  & D3 & 16.1\err{6.1} & 18.0\err{6.5} & 44.7\err{21.7} & 35.4\err{17.4} & 6.3\err{3.9} & 4.0\err{2.8}\\
CMP & D1 & 26.1\err{9.2} & 28.0\err{9.1} & 54.7\err{21.8} & 40.3\err{18.2} & 6.5\err{3.0} & 3.8\err{3.5}\\
CMP & D2 & 21.4\err{6.9} & 22.8\err{10.3} & 61.4\err{19.0} & 36.6\err{17.1} & 6.5\err{3.8} & 4.4\err{2.5}\\
CMP & D3 & 15.3\err{5.2} & 18.4\err{7.3} & 44.9\err{17.2} & 34.2\err{14.7} & 6.8\err{3.3} & 4.5\err{2.7}\\
\midrule
\multicolumn{8}{l}{\textbf{Form 1}}\\
A  & D1 & 24.6\err{15.5} & 26.7\err{14.0} & 69.0\err{19.6} & 39.0\err{20.6} & 7.4\err{6.2} & 8.2\err{6.3}\\
SO & D1 & 9.4\err{8.5} & 12.1\err{12.1} & 62.3\err{31.9} & 42.6\err{24.0} & 9.3\err{6.6} & 9.8\err{5.6}\\
SO & D2 & 11.5\err{6.9} & 12.4\err{6.6} & 45.2\err{23.6} & 48.6\err{17.7} & 8.6\err{5.1} & 10.1\err{4.4}\\
CO & D1 & 13.6\err{9.6} & 15.0\err{12.2} & 61.6\err{23.9} & 42.8\err{19.5} & 5.0\err{5.0} & 5.4\err{5.3}\\
CO & D2 & 12.2\err{5.9} & 12.4\err{8.7} & 57.2\err{16.7} & 48.8\err{15.7} & 6.5\err{3.8} & 7.4\err{4.5}\\
M  & D1 & 15.0\err{10.0} & 14.3\err{11.1} & 34.7\err{22.6} & 20.6\err{19.5} & 6.6\err{5.0} & 6.9\err{4.8}\\
M  & D2 & 12.1\err{6.9} & 10.3\err{5.7} & 37.5\err{17.6} & 20.9\err{10.9} & 7.0\err{4.0} & 6.2\err{3.6}\\
M  & D3 & 14.2\err{6.2} & 13.5\err{7.0} & 33.3\err{20.4} & 27.6\err{15.4} & 6.7\err{3.5} & 6.3\err{3.6}\\
\bottomrule
\end{tabular}
\end{table*}

\begin{figure*}[h]
    \centering
    \includegraphics[width=0.9\linewidth]{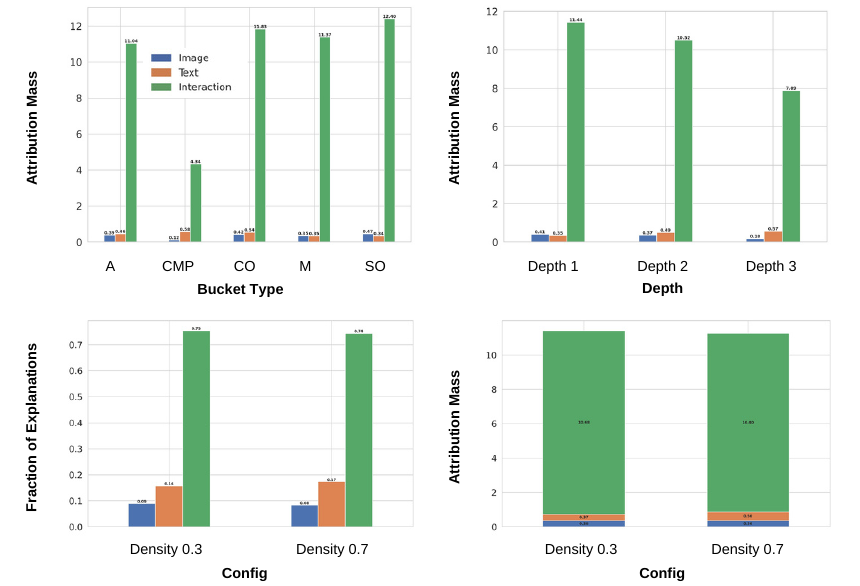}
    \caption{EMAP attribution mass and explanation fractions for $\mathcal{M}_{\text{pure}}$ evaluated on $\mathcal{D}_{\text{pure}}$, segmented by query type, relational depth, and grid density.}
    \label{fig:pureds_puremodel_examples}
\end{figure*}

\begin{figure*}[h]
    \centering
    \includegraphics[width=0.9\linewidth]{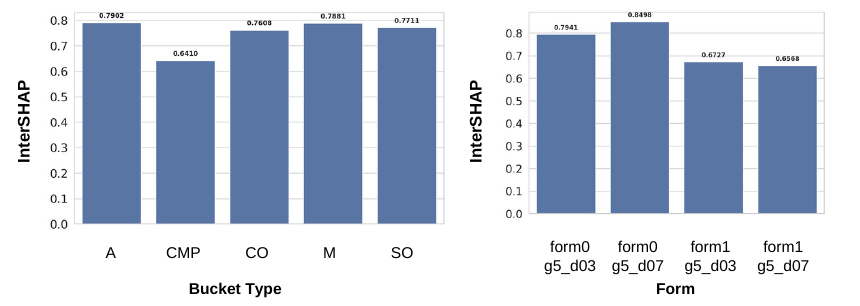}
    \caption{\textbf{InterSHAP score from $\mathcal{M}_{\text{spur}}$ with $\mathcal{D}_{\text{spur}}$.}}
    \label{fig:intershap_spur}
\end{figure*}

\begin{figure*}[h]
    \centering
    \includegraphics[width=0.9\linewidth]{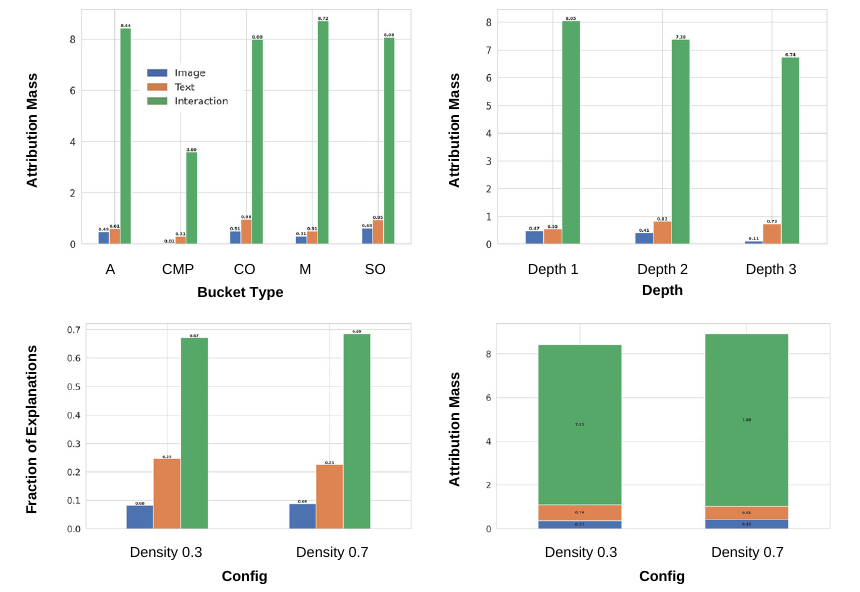}
    \caption{EMAP attribution mass and explanation fractions for $\mathcal{M}_{\text{spur}}$ evaluated on $\mathcal{D}_{\text{spur}}$, segmented by query type, relational depth, and grid density.}
    \label{fig:spurds_spurmodel_examples}
\end{figure*}

\begin{figure*}[h]
    \centering
    \includegraphics[width=0.9\linewidth]{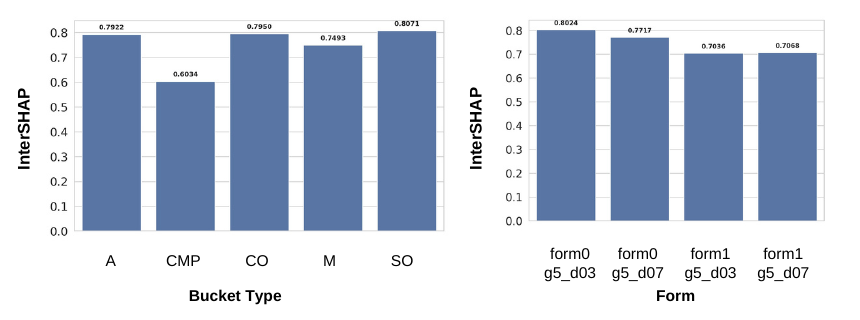}
    \caption{\textbf{InterSHAP scores on $\mathcal{M}_{\text{pure}}$ with $\mathcal{D}_{\text{pure}}$.}}
    \label{fig:intershap_pure}
\end{figure*}

\end{document}